\documentclass[letterpaper]{article} 
\usepackage{aaai24}  
\usepackage{times}  
\usepackage{helvet}  
\usepackage{courier}  
\usepackage[hyphens]{url}  
\usepackage{graphicx} 
\usepackage{natbib}  
\usepackage{caption} 
\usepackage{algorithm}
\usepackage{subfigure}
\usepackage{algorithmic}
\usepackage{bm}
\usepackage{subcaption}
\usepackage{subfigure}
\usepackage{amsthm}
\usepackage{amsmath}
\usepackage{amsfonts}
\usepackage{amssymb}
\usepackage{subfig}
\newtheorem{theorem}{Theorem}
\newtheorem{lemma}{Lemma}
\newtheorem{proposition}{Proposition}
\newtheorem{assumption}{Assumption}
\newtheorem{definition}{Definition}
\allowdisplaybreaks

\usepackage{newfloat}
\usepackage{listings}
\DeclareCaptionStyle{ruled}{labelfont=normalfont,labelsep=colon,strut=off} 
\floatstyle{ruled}
\newfloat{listing}{tb}{lst}{}
\floatname{listing}{Listing}
\title{Stochastic Bayesian Optimization with Unknown Continuous Context Distribution via Kernel Density Estimation}
\author{
    Xiaobin Huang, Lei Song, Ke Xue, Chao Qian\thanks{Chao Qian is the corresponding author.}
}
\affiliations{
    National Key Laboratory for Novel Software Technology, Nanjing University, Nanjing 210023, China\\ School of Artificial Intelligence, Nanjing University, Nanjing 210023, China\\
    \{huangxb, songl, xuek, qianc\}@lamda.nju.edu.cn
}

\usepackage{bibentry}

\begin{document}

\maketitle

\begin{abstract}
Bayesian optimization (BO) is a sample-efficient method and has been widely used for optimizing expensive black-box functions. Recently, there has been a considerable interest in BO literature in optimizing functions that are affected by context variable in the environment, which is uncontrollable by decision makers. In this paper, we focus on the optimization of functions' expectations over continuous context variable, subject to an unknown distribution. To address this problem, we propose two algorithms that employ kernel density estimation to learn the probability density function (PDF) of continuous context variable online. The first algorithm is simpler, which directly optimizes the expectation under the estimated PDF. Considering that the estimated PDF may have high estimation error when the true distribution is complicated, we further propose the second algorithm that optimizes the distributionally robust objective. Theoretical results demonstrate that both algorithms have sub-linear Bayesian cumulative regret on the expectation objective. Furthermore, we conduct numerical experiments to empirically demonstrate the effectiveness of our algorithms.
\end{abstract}

\section{Introduction}

Bayesian optimization (BO)~\cite{bosurvey1, bosurvey2} is a popular and sample-efficient method for optimizing expensive black-box functions. BO has shown excellent performance in various fields, such as chemical molecular design~\cite{autochem, chemical2}, neural architecture search~\cite{mlhp,nasbo,mctsvs}, and hyper-parameter tuning~\cite{boalp,bopp}. The typical process of BO involves approximating the objective function by a Gaussian process (GP) surrogate model~\cite{gp}, and then selecting the most valuable point for evaluation by optimizing an acquisition function based on the posterior of the surrogate model.

In some practical scenarios, uncontrollable context variable from the environment can impact the objective function, such as customer demand in inventory management~\cite{inventory2,inventory}, bid-ask spread and borrowing cost in portfolio optimization~\cite{portfolio,risk1}, and temperature in crop size optimization~\cite{drbo_wcs}. There has been some BO literature taking context variable into account with different objectives. Robust optimization aims to find solutions that perform well in the worst-case scenario~\cite{ro2,ro}, while stochastic optimization (SO) focuses on finding solutions that perform well in expectation~\cite{modified_branin, bointergrand2,bouu,stochasitc_bandit, bointergrand}. Risk optimization considers risk measures such as mean-variance~\cite{mean_variance}, value at risk (VaR)~\cite{risk1, risk2} or conditional VaR~\cite{risk1}. However, robust optimization ignores the distribution information of the context, and most existing works on SO and risk optimization assume that the distribution of context is known. 

Distributionally robust BO (DRBO)~\cite{drbo,drbqo,drbochance,drbo_wcs} is proposed to address problems with unknown context distribution by optimizing the worst expectation over a set of distributions. Existing DRBO works focus on finite context. When the context variable is in a continuous space, which is common (e.g., temperature in crop size optimization, and energy output in wind power prediction~\cite{drbo_wcs}) in practice, they usually discretize the space. However, the computational complexity of the inner convex optimization in DRBO is at least cubed to the size $|\mathcal{C}|$ of context space $\mathcal{C}$~\cite{drbo_wcs}. Smaller $|\mathcal{C}|$ leads to poor approximation to the expectation over the continuous space, thus poor performance, while larger $|\mathcal{C}|$ leads to unacceptable computational complexity. While~\citeauthor{drbo_wcs}~\shortcite{drbo_wcs} have developed a method based on fast worst case sensitivity to efficiently approximate and accelerate the inner optimization problem, it suffers from linear regret due to approximation errors.~\citeauthor{drbqo}~\shortcite{drbqo} used Lagrange multipliers to accelerate the optimization, but it is limited to the simulator setting, where the decision makers can select the context.

In this paper, we consider the problem of maximizing the SO objective $\max_{\bm{x}}\mathbb{E}_{\bm{c}\sim  p(\bm{c})}[f(\bm{x},\bm{c})]$ over the decision variable $\bm x\in \mathcal X$, where $f$ is a black-box function, and the distribution $p$ of context variable $\bm c\in \mathcal C$ is continuous and unknown. The context is observable after making decisions. To avoid the drawbacks of discretization, we propose two algorithms to directly address this problem. The first algorithm employs kernel density estimation (KDE) to estimate the unknown context distribution and maximizes the objective function $f$'s expectation under the estimated PDF, which is simple and time-efficient. Considering that the estimated PDF may have high estimation error when the true distribution is complicated, we propose the second algorithm, which also uses KDE for PDF estimation but maximizes a distributionally robust objective function, i.e., optimizes the worst-case expectation across a set of distributions around the estimated one. We provide theoretical analyses for both algorithms, proving that they have sub-linear Bayesian cumulative regret on the SO objective. The experiments on synthetic functions and two real-world optimization tasks (i.e., newsvendor problem and portfolio optimization) demonstrate that our proposed algorithms achieve better performance.

\section{Background}
\label{background}

\subsection{Bayesian Optimization}
BO is a popular algorithm for black-box optimization, which consists of two main components: a surrogate model and an acquisition function. GP~\cite{gp} is the most commonly used surrogate model. The function $f$ is assumed to be a sample path from a GP, denoted as $\mathcal{GP}(\bm{0}, k(\cdot,\cdot))$, where $\bm{0}$ is the prior mean and $k(\cdot,\cdot)$ is a kernel function. Given observed data set $\mathcal{D}_{t-1} = \{ \left(\bm{x}_i, \bm{c}_i,y_i\right)\}_{i=1}^{t-1}$, where $y_i = f(\bm{x}_i,\bm{c}_i)+\epsilon_i$ is the noisy observation and $\epsilon_i\sim\mathcal{N}(0,\sigma^2)$, we can calculate the posterior distribution of the function $f\mid\mathcal{D}_{t-1}\sim\mathcal{GP}(\mu_{t}(\bm{x},\bm{c}), k_{t}((\bm{x},\bm{c}),(\bm{x}',\bm{c}')))$, where the posterior mean $\mu_{t}(\bm{x},\bm{c})=$ $\bm{k}_{t-1}(\bm{x},\bm{c})^\top(\mathbf{K}_{t-1}+\sigma^2\mathbf{I})^{-1}\bm{y}_{t-1}$, and posterior covariance $k_{t}((\bm{x},\bm{c}),(\bm{x}',\bm{c}'))=k((\bm{x},\bm{c}),(\bm{x}',\bm{c}'))-\bm{k}_{t-1}(\bm{x},$ $\bm{c})^\top(\mathbf{K}_{t-1}+\sigma^2\mathbf{I})^{-1}\bm{k}_{t-1}(\bm{x}',\bm{c}')$. Here, $\bm{k}_{t-1}(\bm{x},\bm{c})$ $= \left[k\left((\bm{x}_i,\bm{c}_i), (\bm{x},\bm{c})\right)\right]_{i=1,\dots,t-1}^\top$, $\mathbf{K}_{t-1}\in\mathbb{R}^{(t-1)\times(t-1)}$ is the positive semi-definite kernel matrix with $[\mathbf{K}_{t-1}]_{ij}=k\left((\bm{x}_i,\bm{c}_i),(\bm{x}_j,\bm{c}_j)\right)$, and $\bm{y}_{t-1} =[y_1,\dots,y_{t-1}]^{\top}$.

Based on the posterior distribution obtained from GP, various acquisition functions can be used to determine the next query point, e.g., Probability of Improvement (PI)~\cite{POI}, Expected Improvement (EI)~\cite{ei} and Upper Confidence Bound (UCB)~\cite{gpucb}. In this work, we use the UCB acquisition function for both our algorithms, defined as $\text{ucb}_t(\bm{x},\bm{c})=\mu_t(\bm{x},\bm{c})+\sqrt{\beta_t}\sigma_t(\bm{x},\bm{c})$, where $\beta_t$ is a hyper-parameter to balance the exploitation and exploration. We use $\sigma_t^2(\bm{x},\bm{c})=k_{t}((\bm{x},\bm{c}),(\bm{x},\bm{c}))$ to represent the posterior variance at $(\bm x, \bm c)$. Note that under deterministic environments, the context variable $\bm c \in \mathcal{C}$ can be neglected.   

\subsection{Stochastic Bayesian Optimization}

In real-world problems, the objective function may be affected by context variable, which is uncontrollable by the decision makers. The problem can be formalized as a black-box function $f(\bm{x},\bm{c})$ over a convex and compact domain $\mathcal{X} \times \mathcal{C} \subset \mathbb{R}^{D_x} \times \mathbb{R}^{D_c}$, where $\mathcal{X}$ is a $D_x$-dimensional decision space controlled by decision makers, and $\mathcal{C}$ is a $D_c$-dimensional context space controlled by environment. In this paper, we consider the setting that $\mathcal{C}$ is continuous. At iteration $t$, a decision $\bm{x}_t$ is made, followed by the observation of a context $\bm{c}_t \sim p(\bm{c})$ provided by the environment and observed by the decision maker. Note that the context distribution $p(\bm{c})$ is unknown here. Next, we observe the noisy evaluation $y_t = f(\bm{x}_t,\bm{c}_t)+\epsilon_t$, where $\epsilon_t\sim\mathcal{N}(0, \sigma^2)$. We consider the SO setting aiming to identify the optimum $\bm{x}^* \in \mathop{\arg\max}_{\bm{x}\in\mathcal{X}}\mathbb{E}_{\bm{c}\sim p(\bm{c})}[f(\bm{x},\bm{c})]$. 
Given the evaluation budget $T$, the goal is to minimize the cumulative regret of SO objective, i.e.,
\begin{equation}
    R_T :=\sum_{t=1}^T\left( \mathbb{E}_{\bm{c}\sim p(\bm{c})}[f(\bm{x}^*,\bm{c})]-\mathbb{E}_{\bm{c}\sim p(\bm{c})}[f(\bm{x}_t,\bm{c})]  \right) .\label{stochastic_cumulative_regret}
\end{equation}

There has been plentiful research considering context variable in the BO literature. For instance,~\citeauthor{bo_context}~\shortcite{bo_context} considered the case that the context $\bm{c}_t$ is given before decision. When the context cannot be known beforehand, numerous approaches have been proposed with different optimization objectives.

\textbf{Robust optimization} considers a worst-case objective, formulated as $\max_{\bm{x}}\min_{\bm{c}\in\mathcal{C}}f(\bm{x},\bm{c})$, and has been studied in~\cite{ro2,ro}. However, robust optimization is too pessimistic, ignoring the distribution information of context.

\textbf{Stochastic optimization (SO)} considers an average-case optimization, i.e., $\max_{\bm{x}}\mathbb{E}_{\bm{c}\sim p(\bm{c})}[f(\bm{x},\bm{c})]$. A special case of SO is optimizing the expectation under input perturbation $\max_{\bm{x}}\mathbb{E}_{\bm{c}\sim p(\bm{c})}\left[f(\bm{x}\diamond \bm{c})\right]$, where $\diamond$ denotes the perturbation of the input $\bm x$ by $\bm c$, which has been discussed in~\cite{unscented_bo,bouu,uncertain_input,noisy_input}. The general case of SO has also been studied using different acquisition functions in~\cite{modified_branin, bointergrand2, bouu,stochasitc_bandit, bointergrand}, which, however, assume that the distribution of context is known.

\textbf{Risk optimization} uses risk measures, such as mean-variance~\cite{mean_variance}, value at risk (VaR)~\cite{risk1,risk2} and conditional VaR~\cite{risk1}, as the objectives when dealing with contextual uncertainty. For example,~\citeauthor{risk1}~\shortcite{risk1} and~\citeauthor{risk2}~\shortcite{risk2} studied $\text{VaR}_\delta(\bm{x}):=\sup\{s:\mathbb{P}(f(\bm{x},\bm{c})\geq s)\geq 1-\delta\}$, which measures the risk under a specified level of confidence $1-\delta$. 

\textbf{Distributionally robust optimization (DRO)}. The above methods usually ignore the context distribution or assume the distribution of the context variable is known. When the distribution is unknown, the DRO objective can be adopted, considering the worst-case expectation over a set of distributions. The DRO objective is formulated as $\max_{\bm{x}}\min_{q\in\mathcal{Q}}\mathbb{E}_{\bm{c}\sim q(\bm{c})}[f(\bm{x},\bm{c})]$, where $\mathcal{Q}$ is a given distribution set on the context space $\mathcal{C}$. Different approaches have been proposed to optimize the DRO objective. For instance,~\citeauthor{drbqo}~\shortcite{drbqo} considered the simulator setting where the decision makers can select the context $\bm{c}$, while~\citeauthor{drbochance}~\shortcite{drbochance} considered DRO under chance constraints.~\citeauthor{drbo}~\shortcite{drbo} used maximum mean discrepancy (MMD) to construct the distribution set $\mathcal{Q}$ given a reference distribution. However, they only considered the case where the context space is discrete with a size of $|\mathcal{C}|$, and the inner optimization is a $|\mathcal{C}|$-dimensional optimization problem, which can be computationally expensive when the size $|\mathcal{C}|$ is large.~\citeauthor{drbo_wcs}~\shortcite{drbo_wcs} proposed using worst-case sensitivity for approximation and acceleration of the inner optimization. However, the approximation error can lead to a decrease in the performance, and the regret bound they derived is linear even when the distribution distance $\epsilon_t=0$ in Theorem 4 in~\cite{drbo_wcs}. Continuous DRO has been discussed in~\cite{drbo_phi}, but the algorithms they proposed only hold under certain conditions, which will be discussed in Section~\ref{section_method_and_theory}.

In this paper, we consider the SO objective and assume that the distribution of context variable is continuous and unknown. The setting is similar to the data-driven setting in~\cite{drbo}, which, however, used a discrete context space. 
A similar setting has also been discussed in bandit problems~\cite{context_bandits}, where the context distributions for each arm are unknown and estimated online.

\subsection{Kernel Density Estimation}
\label{section_kde}
Kernel density estimation (KDE)~\cite{kde_1,kde_scott,kde_tutorial} is a non-parametric method used for estimating the probability density function (PDF) of a random variable, and is widely used in machine learning communities due to its flexibility~\cite{kde_application1,kde_application2}. The basic idea of KDE is to estimate the PDF by aggregating the density assigned around each sample. Given the i.i.d. samples $\{\bm{c}_i\}_{i=1}^{t}$ drawn from $p(\bm{c})$, the estimated distribution $\hat p(\bm c)$ of $p(\bm{c})$ is calculated as
\begin{equation}
   \hat{p}(\bm{c})=\sum\nolimits_{i=1}^{t} K\left(\mathbf{H}_t^{-1}(\bm{c}-\bm{c}_i) \right)/(t|\mathbf{H}_t|),
   \label{KDE}
\end{equation}
where $\mathbf{H}_t= \text{diag}([h_t^{(1)}, h_t^{(2)},\dots,h_t^{(D_c)}])$ is a diagonal positive definite bandwidth matrix, $|\mathbf{H}_t|$ and $\mathbf{H}_t^{-1}$ denote the determinant and inverse of $\mathbf{H}_t$, respectively, and $K(\cdot)$ is a kernel function satisfying 
$K(\bm{c}) = K(-\bm{c}), \forall\bm{c}\in\mathcal{C},$ $\int_{\mathcal{C}}\bm c \bm c^{\mathrm{T}} K(\bm c)\, d \bm c=m_2(K)\mathbf{I}_{D_c}$ for some constant $m_2(K)>0$, and $\int_{\mathcal{C}} K(\bm{c})\, d\bm{c}=1.$

Besides flexibility, KDE has good theoretical convergence properties for different error functions, e.g., uniform error~\cite{uniform_error}, $\ell_1$ error~\cite{kde_l1} and mean integrated square error (MISE)~\cite{mise_multi,kde_tutorial}. In this work, we primarily focus on MISE, which is one of the most well-known error measurements. Lemma~\ref{l2_bound_lemma} gives an upper bound on the MISE. It can be shown that by choosing $h_t^{(i)} \, \forall i$ to be of order $\Theta\left(t^{-1/(4+D_c)}\right)$, the MISE can be upper bounded by $\mathcal{O}\left(t^{-4/(D_c+4)}\right)$~\cite{mise_multi}, which will play a crucial role in deriving the regret bound for our algorithms in Section~\ref{section_method_and_theory}.

\begin{lemma}[\citeauthor{mise_multi}, \citeyear{mise_multi}]
Suppose $K(\cdot)$ is a bounded kernel for KDE, and $p(\bm{c})$ is a twice-differentiable PDF over $\mathcal{C}$. Let $J = \int_{\mathcal{C}}(p(\bm{c}) - \hat{p}(\bm{c}))^2\, d\bm{c}$ with $\hat{p}$ defined as Eq.~\eqref{KDE}. Then the MISE $\mathbb{E}[J]$ has an order of
\begin{align}
   \mathcal{O}\left(\frac{1}{t|\mathbf{H}_t|}R(K)
+\frac{1}{4}m_2(K)^2(\textnormal{vec}^{\mathrm{T}} (\mathbf{H}_t^2))\Psi_4(\textnormal{vec} (\mathbf{H}_t^2))\right),
\nonumber
\end{align}
where $R(K)=\int_{\mathcal{C}}K(\bm c)^2 \,d \bm c$, $\textnormal{vec}(\cdot)$ is the vector operator that vectorizes a matrix into a column vector, $\Psi_4=\int_{\mathcal{C}}(\textnormal{vec}(\nabla^2\,p(\bm c)))(\textnormal{vec}^{\mathrm{T}}( \nabla^2\,p(\bm c)))\, d\bm{c}$ is a $D_c^2 \times D_c^2$ matrix of integrated second order partial derivatives of the PDF $p$, and the expectation $\mathbb{E}[J]$ is taken over the randomness of samples $\{\bm{c}_i\}_{i=1}^{t}$ from $p(\bm{c})$.
\label{l2_bound_lemma}
\end{lemma}

\section{Stochastic Bayesian Optimization \\with Kernel Density Estimation}
\label{section_method_and_theory}

We propose two algorithms to optimize the SO objective with unknown continuous context PDF. The main idea is to estimate the PDF of the context online using KDE. The only difference lies in the design and optimization of their acquisition functions. The first algorithm, SBO-KDE, is directly based on an acquisition function of SO objective, which takes the expectation under the PDF estimated by KDE. The second one, DRBO-KDE, is based on a distributionally robust acquisition function, which accounts for distribution discrepancy between the true and estimated PDF, by taking the worst-case expected value in the distribution set centered around the estimated PDF.

\subsection{SBO-KDE}
SBO-KDE optimizes the SO objective $\mathbb{E}_{\bm c \sim \hat{p}_t(\bm{c})}[f(\bm{x},\bm{c})]$ directly by using the estimated distribution $\hat{p}_t$ by KDE. The key component of the algorithm is the acquisition function $\alpha_t(\bm{x}) = \mathbb{E}_{\bm{c}\sim \hat{p}_t(\bm{c})}[\text{ucb}_t(\bm{x},\bm{c})]$, which can be interpreted as the expectation of the UCB acquisition function under the estimated distribution $\hat{p}_t$. The algorithm procedure is described in Algorithm~\ref{kdesbo_algorithm}. In line~1, the initial data set $\mathcal{D}_{n_0}=\{(\bm{x}_i,\bm{c}_i, y_i)\}_{i=1}^{n_0}$ is sampled using Sobol sequence~\cite{sobol}, where $n_0$ is the number of initial points. The optimization procedure is shown in lines~2--8. At iteration $t$, with the observed context $\mathcal{C}_{t-1}=\{\bm{c}_i\}_{i=1}^{t-1}$, we estimate the unknown context distribution $p(\bm{c})$ using KDE in line~3, and the estimated distribution is denoted as $\hat{p}_t(\bm{c})$. Then we fit a GP model based on the current data set $\mathcal{D}_{t-1}$ in line~4. With the estimated PDF and the posterior information, we optimize the acquisition function using sample average approximation (SAA) to get the next query point $\bm{x}_t$ in line~5. SAA uses the average of sample values to estimate the value of acquisition function $\alpha_t(\bm{x})$. When evaluating $\bm x_t$, the context $\bm{c}_t$ provided by the environment and the noisy function value $y_t$ is observed in line~6. The data set is then augmented with the new triple $(\bm{x}_t, \bm{c}_t, y_t)$ in line~7. The whole process is repeated for $T-n_0$ iterations.

\begin{algorithm}[h]
\caption{SBO-KDE}
\label{kdesbo_algorithm}
\textbf{Parameters}: number $n_0$ of initial points, budget $T$ \\
{\textbf{Process}:}
\begin{algorithmic}[1] 
\STATE Obtain the initial data set $\mathcal{D}_{n_0}=\{(\bm{x}_i,\bm{c}_i, y_i)\}_{i=1}^{n_0}$ and context $\mathcal{C}_{n_0}=\{\bm{c}_i\}_{i=1}^{n_0}$ using Sobol sequence;
\FOR{$t=n_{0}+1$ to $T$} 
\STATE Use KDE to obtain $\hat{p}_t$ based on $\mathcal{C}_{t-1}=\{\bm{c}_i\}_{i=1}^{t-1}$;
\STATE Fit a GP model using $\mathcal{D}_{t-1} = \{(\bm{x}_i,\bm{c}_i, y_i)\}_{i=1}^{t-1}$;
\STATE Optimize $\bm{x}_t = {\arg\max}_{\bm{x}\in\mathcal{X}}\mathbb{E}_{\bm{c}\sim \hat{p}_t(\bm{c})}[\text{ucb}_t(\bm{x},\bm{c})]$ using SAA;
\STATE Evaluate $\bm{x}_t$, and then observe $\bm{c}_t \sim p(\bm{c})$ and $y_t = f(\bm{x}_t, \bm{c}_t)+\epsilon_t$;
\STATE $\mathcal{D}_{t} = \mathcal{D}_{t-1} \cup \{(\bm{x}_t, \bm{c}_t, y_t)\}$
\ENDFOR
\end{algorithmic}
\end{algorithm}

To optimize the acquisition function $\alpha_t(\bm{x})$, any technique from traditional SO can be employed. In this work, we adopt the SAA method~\cite{saa_convergence,saa_guide}, which optimizes the average function value of Monte Carlo samples. Specifically, we draw $M$ samples $\{\hat{\bm{c}}_i\}_{i=1}^M$ from $\hat{p}_t(\bm{c})$ and estimate the value of acquisition function as $\hat{\alpha}_t^M(\bm{x})=\frac{1}{M}\sum_{i=1}^M\text{ucb}_t(\bm{x},\hat{\bm{c}}_i)$. We then 
optimize $\bm{x}_t = \arg\max_{\bm{x}\in\mathcal{X}}\hat{\alpha}_t^M(\bm{x})$ using L-BFGS~\cite{lbfgs}. SAA is a popular technique for optimizing acquisition functions in BO~\cite{botorch, risk1}, due to its exponential convergence property in Proposition~\ref{saa_convergence_proposition}. The property can be derived based on the theoretical results from~\cite{saa_convergence,botorch}, and the detailed proof is provided in Appendix~\ref{theorey_saa_appendix}. The exponential convergence rate of SAA enables us to obtain good acquisition function optimization quality of $\alpha_t(\bm{x})$. 

\begin{proposition}
Suppose that (\romannumeral 1) $\{\hat{\bm{c}}_i\}_{i=1}^M$ is i.i.d., and (\romannumeral 2) $f$ is a GP with continuously differentiable prior mean and kernel function. Then, $\forall \delta >0$, there exist $Q<\infty$ and $\eta>0$ such that $\mathbb{P}(\text{dist}(\hat{\bm{x}}_M^*, \mathcal{X}^*_t)>\delta)\leq Qe^{-\eta M}$ for all $M\geq 1$, where dist$(\hat{\bm{x}}_M^*, \mathcal{X}^*_t)=\inf_{\bm x\in \mathcal{X}^*_t}\|\bm x- \hat{\bm{x}}_M^*\|_2$, $\hat{\bm{x}}_M^* \in \arg\max_{\bm{x}\in\mathcal{X}}\hat{\alpha}_t^M(\bm{x})$ and $\mathcal{X}^*_t=\arg\max_{\bm x\in\mathcal{X}}\alpha_t(\bm{x})$.
\label{saa_convergence_proposition}
\vspace{-1em}
\end{proposition}
 
Although the acquisition function $\alpha_t(\bm{x})$ uses an estimated context distribution $\hat{p}_t$ by KDE, we prove that the algorithm SBO-KDE can still achieve a sub-linear bound under the true distribution $p$, as shown in Theorem~\ref{main_theorem_1}. The sub-linear bound is on the commonly used Bayesian cumulative regret (BCR)~\cite{BCR,BCR_TS,drbqo} in Definition~\ref{def_BCR}, which is actually the expectation of $R_T$ in Eq.~\eqref{stochastic_cumulative_regret}.

\begin{definition}[BCR]\label{def_BCR}
    Let $r_t:=\mathbb{E}_{\bm{c}\sim p}[f(\bm{x}^*,\bm{c})]-\mathbb{E}_{\bm{c}\sim p}[f(\bm{x}_t,\bm{c})]$ denote the regret at iteration $t$, where $\bm{x}^* \in \arg \max_{x\in\mathcal{X}} \mathbb{E}_{\bm{c}\sim p}[f(\bm{x},\bm{c})]$ is the optimum. Then, the Bayesian cumulative regret is defined as 
    \begin{align}
    \text{BCR}(T) &:= \mathbb{E}\left[\sum_{t=1}^T r_t\right]        \label{BCR}
    \\
    &=\mathbb{E}\left[\sum_{t=1}^T\left( \mathbb{E}_{\bm{c}\sim p}[f(\bm{x}^*,\bm{c})]-\mathbb{E}_{\bm{c}\sim p}[f(\bm{x}_t,\bm{c})]  \right) \right],
    \nonumber
    \end{align}
    where the outer expectation is taken over the GP $f$, the randomness of samples from $p(\bm{c})$ and the observation noise $\epsilon_t$.
\end{definition}

As in~\cite{gpucb}, we assume that the input space $\mathcal{Z} = \mathcal{X} \times \mathcal{C} \subset [0,r]^{D_x+D_c}=[0,r]^{D}$ is convex and compact, and $f$ satisfies the following Lipschitz assumption.

\begin{assumption}
The function $f$ is a GP sample path from $\mathcal{GP}(\bm{0}, k(\cdot,\cdot))$ with $k(\bm{z},\bm{z}')\leq 1,$ $\forall \bm{z},\bm{z}' \in \mathcal{Z}$. Let $[D]=\{1,2,\dots,D\}$. For some $a,b >0$, $\forall L>0$, the partial derivatives of $f$ satisfy
\begin{equation}
   \forall i \in [D],\; \mathbb{P}(\sup\nolimits_{\bm{z} \in \mathcal{Z}}\mid \partial f(\bm{z})/\partial z_i \mid > L) \leq ae^{-(L/b)^2}.
    \nonumber
\end{equation}
\label{lipschitz}
\vspace{-1em}
\end{assumption}
Theorem~\ref{main_theorem_1} provides an upper bound on $\text{BCR}(T)$ of SBO-KDE. Ignoring the log factors from $\beta_T\gamma_T$, compared with the bound of stochastic BO (SBO) with $\mathcal{O}(T^{1/2})$ in~\cite{stochasitc_bandit}, our bound increases to $\mathcal{O}(T^{(2+D_c)/(4+D_c)})$, which comes from the estimation error of KDE under the unknown context distribution setting. However, the bound is still sub-linear, i.e., $\lim_{T\rightarrow\infty}\text{BCR}(T)/T=0$. 


\begin{theorem}
Let $\beta_t = 2\log(t^2/\sqrt{2\pi}) +2D_x\log(t^2D_xabr\sqrt{\pi}/2)$. With the underlying PDF $p(\bm{c})$ satisfying the condition in Lemma~\ref{l2_bound_lemma}, $\hat{p}_t(\bm{c})$ defined as Eq.~\eqref{KDE} and $h_t^{(i)}=\Theta\left(t^{-1/(4+D_c)}\right)\forall i\in [D_c]$, the BCR of SBO-KDE satisfies
\begin{align}\label{eq-BCR-SBO}
\textnormal{BCR}(T)\leq&\frac{\pi^2}{3} + \sqrt{\beta_T\gamma_TC_2}\left(\sqrt{T^{D_c/(4+D_c)}}+\sqrt{T}\right)
\nonumber
\\
& + 2C_1T^{\frac{2+D_c}{4+D_c}},
\end{align}
where $C_1,C_2>0$ are constants, $\gamma_T=\max_{|\mathcal{D}|=T} I(\bm{y}_\mathcal{D}, $ $ \bm{f}_\mathcal{D})$, $I(\cdot,\cdot)$ is the information gain, and $\bm{y}_\mathcal{D}, \bm{f}_\mathcal{D}$ are the noisy and true observations of a data set $\mathcal{D}$, respectively.
\label{main_theorem_1}
\end{theorem}

We present only a proof sketch here, and the detailed proof can be found in Appendix~\ref{theorey_kdesbo_appendix}. The proof idea is to decompose the instantaneous regret $r_t$ into the following three terms using the estimated PDF $\hat{p}_t$. Specifically,  $r_t=\mathbb{E}_{\bm{c}\sim p}[f(\bm{x}^*,\bm{c})]-\mathbb{E}_{\bm{c}\sim p}[f(\bm{x}_t,\bm{c})]=(\mathbb{E}_{\bm{c}\sim p}[f(\bm{x}^*,\bm{c})]-\mathbb{E}_{\bm{c}\sim \hat{p}_t}[f(\bm{x}^*,\bm{c})])+(\mathbb{E}_{\bm{c}\sim \hat{p}_t}[f(\bm{x}^*,\bm{c})]-\mathbb{E}_{\bm{c}\sim \hat{p}_t}[f(\bm{x}_t,\bm{c})])+(\mathbb{E}_{\bm{c}\sim \hat{p}_t}[f(\bm{x}_t,\bm{c})]-\mathbb{E}_{\bm{c}\sim p}[f(\bm{x}_t,\bm{c})])$. The first and last terms can be upper bounded using the error bound between the true distribution $p$ and the estimated distribution $\hat{p}_t$ as given in Lemma~\ref{l2_bound_lemma}. The second term $\mathbb{E}_{\bm{c}\sim \hat{p}_t}[f(\bm{x}^*,\bm{c})]-\mathbb{E}_{\bm{c}\sim \hat{p}_t}[f(\bm{x}_t,\bm{c})]$ is the UCB regret under expectation over $\hat{p}_t$, which can be upper bounded using the similar idea as in~\cite{BCR_TS}. That is, using the fact that $\bm{x}_t=\arg\max_{\bm{x}\in\mathcal{X}}\mathbb{E}_{\bm{c}\sim \hat{p}_t}[\text{ucb}_t(\bm{x}_t,\bm{c})]$, we further decompose the second term as follows: $\mathbb{E}_{\bm{c}\sim \hat{p}_t}[f(\bm{x}^*,\bm{c})]-\mathbb{E}_{\bm{c}\sim \hat{p}_t}[f(\bm{x}_t,\bm{c})] \leq (\mathbb{E}_{\bm{c}\sim \hat{p}_t}[f(\bm{x}^*,\bm{c})]-\mathbb{E}_{\bm{c}\sim \hat{p}_t}[\text{ucb}_t(\bm{x}^*,\bm{c})])+(\mathbb{E}_{\bm{c}\sim \hat{p}_t}[\text{ucb}_t(\bm{x}_t,\bm{c})]-\mathbb{E}_{\bm{c}\sim \hat{p}_t}[f(\bm{x}_t,$ $\bm{c})])$, which can be upper bounded under the posterior information of GP. By summing up the upper bound on $r_t$ from $t=1$ to $T$, the upper bound on $\textnormal{BCR}(T)$, i.e., Eq.~\eqref{eq-BCR-SBO}, can be derived. 

\subsection{DRBO-KDE}\label{sec-DRBO-KDE}

Considering that the true PDF of context variable in practice may be complicated, leading to a non-negligible distribution discrepancy between the estimated PDF and the true PDF, we further propose the second algorithm, DRBO-KDE. The optimization objective is the worst-case expectation $\min_{q\in\mathcal{B}(\hat{p}_t,\delta_t)}\mathbb{E}_{\bm{c}\sim q}[f(\bm{x},\bm{c})]$ within a distribution set $\mathcal{B}(\hat{p}_t,\delta_t)=\{q: d(q,\hat{p}_t)\leq \delta_t\}$, where $d(\cdot, \cdot)$ is a distance measure over the distribution space, and $\delta_t>0$ is the radius of the ball centered around $\hat{p}_t$. We choose the total variation metric~\cite{TV_definition} here, which is a type of $\phi$-divergence~\cite{dro_phi}. The total variation between two distributions $Q$ and $P$, with PDFs $dQ=q$ and $dP=p$ respectively, is defined as $d_{TV}(Q,P) = d(q,p)= \int_{\mathcal{C}}p(\bm{c})\phi\left(\frac{q(\bm{c})}{p(\bm{c})}\right)\, d\bm{c}$ with $\phi(x)=|x-1|$ and $Q\ll P$. The total variation metric can be thought of as the $\ell_1$ distance between two PDFs when $Q\ll P$.

While DRBO with $\phi$-divergence over continuous space has been discussed in~\cite{drbo_phi}, the proposed algorithms under $\chi^2$-divergence and total variation hold only if the variance of the function value is sufficiently high, which can be derived from the result of Theorem 1 in~\cite{dro_variance}. Therefore, we re-develop an algorithm for DRBO over continuous context space using total variation. Proposition~\ref{dro_transformation_prop} derives an equivalent form of the DRO objective under total variation, which transforms the inner convex minimization problem, which has infinite dimensions, into a two-dimensional convex SO problem. This transformation has been commonly utilized in DRO literature~\cite{dro_phi, drosurvey}. For the sake of completeness, we also provide a detailed derivation in Appendix~\ref{theorey_kdedrbo_appendix}. 

\begin{proposition}[DRO under Total Variation]
    Given a bounded function $f(\bm{x},\bm{c})$ over $\mathcal{X} \times \mathcal{C}$, a radius $\delta>0$ and a PDF $p(\bm{c})$, we have
    \begin{align}
&\min_{q\in\mathcal{B}(p,\delta)}\mathbb{E}_{\bm{c}\sim q(\bm{c})}\left[ f(\bm{x},\bm{c})  \right] 
\label{kdedro_tansformation}
\\
= &\max_{(\alpha,\beta)\in S(f)}\mathbb{E}_{\bm{c}\sim p(\bm{c})}\left[-\beta-\delta\alpha+ \min\{f(\bm{x},\bm{c})+\beta, \alpha\} \right],     \nonumber
    \end{align}
where $\mathcal{B}(p,\delta)=\{q: d(q,p)\leq \delta\}$, and $S(f):=\{ (\alpha,\beta): \beta\in \mathbb{R}, \alpha\geq 0, \alpha+\beta\geq -\inf\nolimits_{\bm{c}\in\mathcal{C}}f(\bm{x},\bm{c}) \}.$
\label{dro_transformation_prop}
\end{proposition}

Based on Proposition~\ref{dro_transformation_prop}, we propose DRBO-KDE as presented in Algorithm~\ref{kdedrbo_algorithm}. The only difference between SBO-KDE and DRBO-KDE is the definition and optimization of the acquisition function in line~5. For DRBO-KDE, the acquisition function is defined as $\alpha_t(\bm{x}) = \min_{q\in\mathcal{B}(\hat{p}_t,\delta_t)}\mathbb{E}_{\bm{c}\sim q(\bm{c})}\left[ \text{ucb}_t(\bm{x},\bm{c})  \right]$, which is the worst expectation of UCB in the distribution set $\mathcal{B}(\hat{p}_t,\delta_t)$. By applying Proposition~\ref{dro_transformation_prop}, we can equivalently transform the optimization into a two-dimensional SO problem: 
$\alpha_t(\bm{x}) = \max_{(\alpha,\beta)\in S(\text{ucb}_t)}\mathbb{E}_{\bm{c}\sim \hat{p}_t(\bm{c})}[-\beta- \delta_t\alpha+ \min\{\text{ucb}_t(\bm{x},\bm{c})+\beta, \alpha\}].$ 
We also use SAA to optimize the acquisition function, i.e., we sample the Monte Carlo samples $\{\hat{\bm{c}}_i\}_{i=1}^M$ from $\hat{p}_t(\bm{c})$ to find the next query point
\begin{align}\bm{x}_t = \mathop{\arg\max}_{\bm{x}\in\mathcal{X}}\max_{(\alpha,\beta)\in S(\text{ucb}_t)}\frac{1}{M}\sum\nolimits_{i=1}^M(-\beta-\delta_t\alpha
\nonumber
\\
+\min\{\text{ucb}_t(\bm{x},\hat{\bm{c}}_i)+\beta, \alpha\}).\label{eq-dro-total-saa}
\end{align} 
This is a two-stage optimization, where the inner optimization problem is a two-dimensional convex optimization problem which can be solved efficiently, and the outer optimization is solved by using L-BFGS~\cite{lbfgs}. Note that $S(\text{ucb}_t)$ can be calculated by numerical optimization.

\begin{algorithm}[h!]
\caption{DRBO-KDE}
\label{kdedrbo_algorithm}
\textbf{Parameters}: number $n_0$ of initial points, evaluation budget $T$, radius $\delta_t > 0$ \\
{\textbf{Process}:}
\begin{algorithmic}[1] 
\STATE Obtain the initial data set $\mathcal{D}_{n_0}=\{(\bm{x}_i,\bm{c}_i, y_i)\}_{i=1}^{n_0}$ and context $\mathcal{C}_{n_0}=\{\bm{c}_i\}_{i=1}^{n_0}$ using Sobol sequence;
\FOR{$t=n_{0}+1$ to $T$} 
\STATE Use KDE to obtain $\hat{p}_t$ based on $\mathcal{C}_{t-1} = \{\bm{c}_i\}_{i=1}^{t-1}$;
\STATE Fit a GP model using $\mathcal{D}_{t-1} = \{(\bm{x}_i,\bm{c}_i, y_i)\}_{i=1}^{t-1}$;
\STATE Optimize $\bm{x}_t = \arg\max_{\bm{x}\in\mathcal{X}}\min_{q\in\mathcal{B}(\hat{p}_t,\delta_t)}\mathbb{E}_{\bm{c}\sim q(\bm{c})}$ $\left[ \text{ucb}_t(\bm{x},\bm{c})  \right]$ using Eq.~\eqref{kdedro_tansformation} and SAA;
\STATE Evaluate $\bm{x}_t$, and then observe $\bm{c}_t \sim p(\bm{c})$ and $y_t = f(\bm{x}_t, \bm{c}_t)+\epsilon_t$;
\STATE $\mathcal{D}_{t} = \mathcal{D}_{t-1} \cup \{(\bm{x}_t, \bm{c}_t, y_t)\}$
\ENDFOR
\end{algorithmic}
\end{algorithm}

Theorem~\ref{main_theorem_2} gives an upper bound on BCR$(T)$ of DRBO-KDE. Compared with the bound of SBO-KDE in Theorem~\ref{main_theorem_1}, the bound of DRBO-KDE is higher by the additional terms related to $\delta_t$. This is because in our proof, we used the estimated distribution $\hat{p}_t$ to establish the connection between the true distribution $p$ and the distribution set $\mathcal{B}(\hat{p}_t, \delta_t)$ around $\hat{p}_t$. If we could directly build the connection between $p$ and $\mathcal{B}(\hat{p}_t, \delta_t)$, we might obtain a better bound, which is, however, more challenging and is left for future work. Nevertheless, by setting $\delta_t =\mathcal{O}\left(t^{-2/(4+D_c)}\right)$, the bound of DRBO-KDE can also be sub-linear with the same order of $\mathcal{O}\left(T^{(2+D_c)/(4+D_c)}\right)$ as SBO-KDE.

\begin{theorem}
Let $\beta_t = 2\log(t^2/\sqrt{2\pi}) +2D_x\log(t^2D_xabr\sqrt{\pi}/2)$. With the underlying PDF $p(\bm{c})$ satisfying the condition in Lemma~\ref{l2_bound_lemma}, $\hat{p}_t(\bm{c})$ defined as Eq.~\eqref{KDE} and $h_t^{(i)}=\Theta\left(t^{-1/(4+D_c)}\right)\forall i\in [D_c]$, the BCR of DRBO-KDE satisfies
\begin{align}
\textnormal{BCR}(T)\leq&\frac{\pi^2}{3} + \sqrt{\beta_T\gamma_TC_2}
\nonumber\bigg(\sqrt{T^{D_c/(4+D_c)}}+\sqrt{T}\
\\
&+\sqrt{C_3\sum_{t=1}^T\delta_t^2}\bigg) + 2C_1T^{\frac{2+D_c}{4+D_c}} + \sum_{t=1}^T C_4\delta_t,
\nonumber
\end{align}
where $C_1,C_2,C_3,C_4>0$ are constants, $\gamma_T=\max_{|\mathcal{D}|=T} I(\bm{y}_\mathcal{D},  \bm{f}_\mathcal{D})$, $I(\cdot,\cdot)$ is the information gain, and $\bm{y}_\mathcal{D}, \bm{f}_\mathcal{D}$ are the noisy and true observations of a data set $\mathcal{D}$, respectively.
\label{main_theorem_2}
\end{theorem}

We also only present a proof sketch here, and the detailed proof is provided in Appendix~\ref{theorey_kdedrbo_appendix}. The idea is similar to that of Theorem~\ref{main_theorem_1} for SBO-KDE. Specifically, for any function $g$, let $q_{\bm{x}}^g:=\arg\min_{q\in\mathcal{B}(\hat{p}_t,\delta_t)}\mathbb{E}_{\bm{c}\sim q}[g(\bm{x},\bm{c})]$. Then, we decompose the instantaneous regret at iteration $t$ as $r_t = \mathbb{E}_{\bm{c}\sim p}[f(\bm{x}^*,\bm{c})]-\mathbb{E}_{\bm{c}\sim p}[f(\bm{x}_t,\bm{c})]=( \mathbb{E}_{\bm{c}\sim p}[f(\bm{x}^*,\bm{c})]- \mathbb{E}_{\bm{c}\sim \hat{p}_t}[f(\bm{x}^*,\bm{c})])
+
( \mathbb{E}_{\bm{c}\sim \hat{p}_t}[f(\bm{x}^*,\bm{c})]  - \mathbb{E}_{\bm{c}\sim q_{\bm{x}^*}^f}[f(\bm{x}^*,\bm{c})])
+
( \mathbb{E}_{\bm{c}\sim q_{\bm{x}^*}^f}[f(\bm{x}^*,\bm{c})]  - \mathbb{E}_{\bm{c}\sim q_{\bm{x}_t}^f}[f(\bm{x}_t,\bm{c})]   )
+
(  \mathbb{E}_{\bm{c}\sim q_{\bm{x}_t}^f}[f(\bm{x}_t,\bm{c})] $ $-\mathbb{E}_{\bm{c}\sim \hat{p}_t}[f(\bm{x}_t,\bm{c})])+ (\mathbb{E}_{\bm{c}\sim \hat{p}_t}[f(\bm{x}_t,\bm{c})]  - \mathbb{E}_{\bm{c}\sim p}[f(\bm{x}_t,\bm{c})])$. All terms, with the except of the third one, can be bounded using the error bound between the estimated PDF $\hat{p}_t$ and the true distribution $p$ or the radius $\delta_t$ of the distribution
 ball $\mathcal{B}(\hat{p}_t,\delta_t)$. The third term is the UCB regret, but under the DRO objective at each iteration, which can be bounded using the posterior information of GP and the fact that $q_{\bm{x}}^g$ belongs to $\mathcal{B}(\hat{p}_t,\delta_t)$ for any $\bm{x}$ and $g$. 

\section{Experiments}
\label{experiment}

In order to empirically evaluate the effectiveness of SBO-KDE and DRBO-KDE, we conduct numerical experiments on synthetic functions and two real-world problems, i.e., the Newsvendor and portfolio problem. We use five identical random seeds (100--104) for all problems and
methods. The code is available at https://github.com/lamda-bbo/sbokde.

\subsection{Experimental Setting}
We adopt the frequently-used cumulative regret (or reward) in BO literature considering uncertainty~\cite{drbo,drbo_wcs} as performance evaluation metric. The experimental setting of our algorithms and compared baselines are summarized as follows.

\textbf{SBO-KDE}. We choose the Gaussian kernel $K(\bm x)=(2\pi)^{-D_c/2}e^{-\|\bm x\|_2^2}$ for KDE. The bandwidth $h_t^{(i)}=(4/(D_c+2))^{1/(4+D_c)} \hat{\sigma}_t^{(i)}t^{-1/(4+D_c)}$ based on the rule of thumb~\cite{kde_1}, where $\hat{\sigma}_t^{(i)}$ is the standard deviation of the $i$th dimension of observed context.

\textbf{DRBO-KDE}. The radius of the distribution set is set as $\delta_t=t^{-2/(4+D_c)}$, which can guarantee a sub-linear regret as we introduced in Section~\ref{sec-DRBO-KDE}. The kernel and bandwidth for KDE are the same as those used in SBO-KDE.

\textbf{DRBO-MMD}~\cite{drbo} discretizes the continuous context space $\mathcal{C}$, and selects the next query point $\bm{x}_t =\arg\max_{\bm{x}\in\mathcal{X}}\min_{Q\in\mathcal{B}(\hat{P}_t,\delta_t)}\mathbb{E}_{\bm{c}\sim Q}[\text{ucb}_t(\bm{x},\bm{c})]$ with mean maximum discrepancy (MMD) as the distribution distance on the discretized context space $\tilde{\mathcal{C}}$. The reference distribution $\hat{P}_t$ is the empirical distribution in $\tilde{\mathcal{C}}$, and we set $\delta_t=(2+\sqrt{2\log{(1/\gamma)} })/\sqrt{t}$ with $\gamma=0.1$, as suggested by Lemma~3 in~\cite{drbo}. Due to its high computational complexity of the inner convex optimization, we use a small discretization size $|\tilde{C}|=\lceil 100^{1/D_c} \rceil^{D_c}$.

\textbf{DRBO-MMD-MinimaxApprox}~\cite{drbo_wcs} accelerates DRBO-MMD with minimax approximation. Thus, we can use a larger discretization size $|\tilde{C}| = \lceil 1024^{1/D_c} \rceil^{D_c}$.

\textbf{StableOpt}~\cite{ro} selects $\bm{x}_t =\arg$ $\max_{\bm{x}\in\mathcal{X}}\min_{\bm{c}\in\mathcal{C}_t}\text{ucb}_t(\bm{x},\bm{c})$. There is no standard way to choose $\mathcal{C}_t$. Instead of setting $\mathcal{C}_t=\mathcal{C}$, which is overly conservative and assumes worst-case scenario under the full context space, we set each dimension of $\mathcal{C}_t$ to $[\mu_{\bm{c}}^{(i)}-\sigma_{\bm{c}}^{(i)}, \mu_{\bm{c}}^{(i)}+\sigma_{\bm{c}}^{(i)}]$, where $\mu_{\bm{c}}^{(i)}$ and $\sigma_{\bm{c}}^{(i)}$ are the mean and standard deviation of the $i$th dimension of observed context, respectively.
 
\textbf{GP-UCB}~\cite{gpucb} ignores the context variable. That is, it builds GP only on the decision space $\mathcal{X}$ and selects the next query point $\bm{x}_t$ by maximizing the UCB acquisition function. 

Further detailed descriptions and hyper-parameters of these algorithms are provided in Appendix~\ref{method_appendix}.

\subsection{Synthetic Functions}
\begin{figure*}[t]
\centering

  \begin{subfigure}
    \centering
\includegraphics[width=0.24\linewidth]{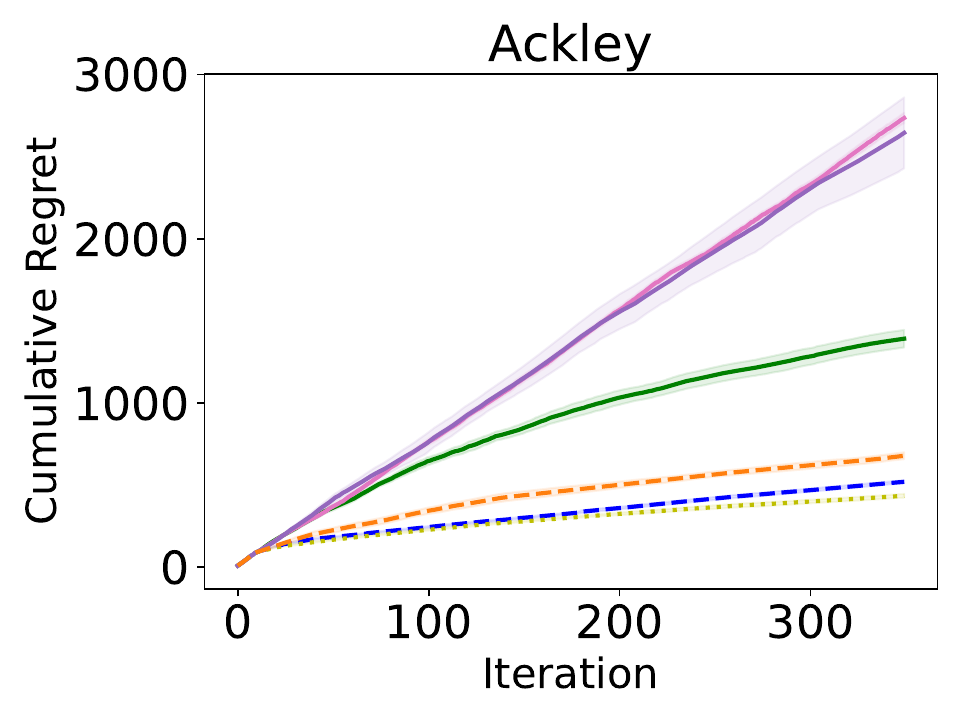}
\includegraphics[width=0.24\linewidth]{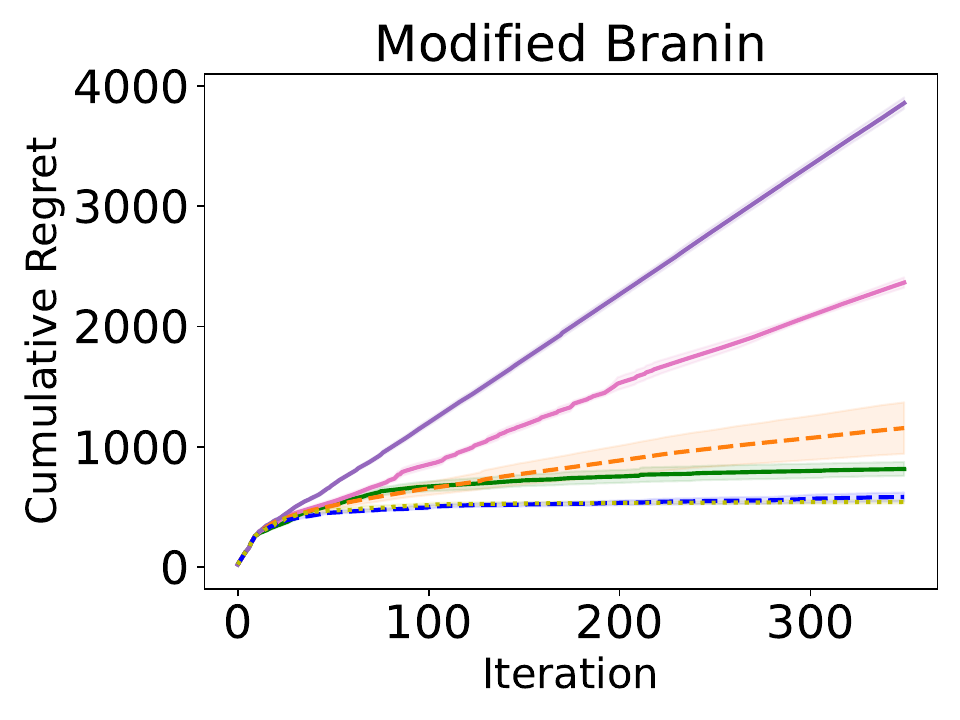}  \includegraphics[width=0.24\linewidth]{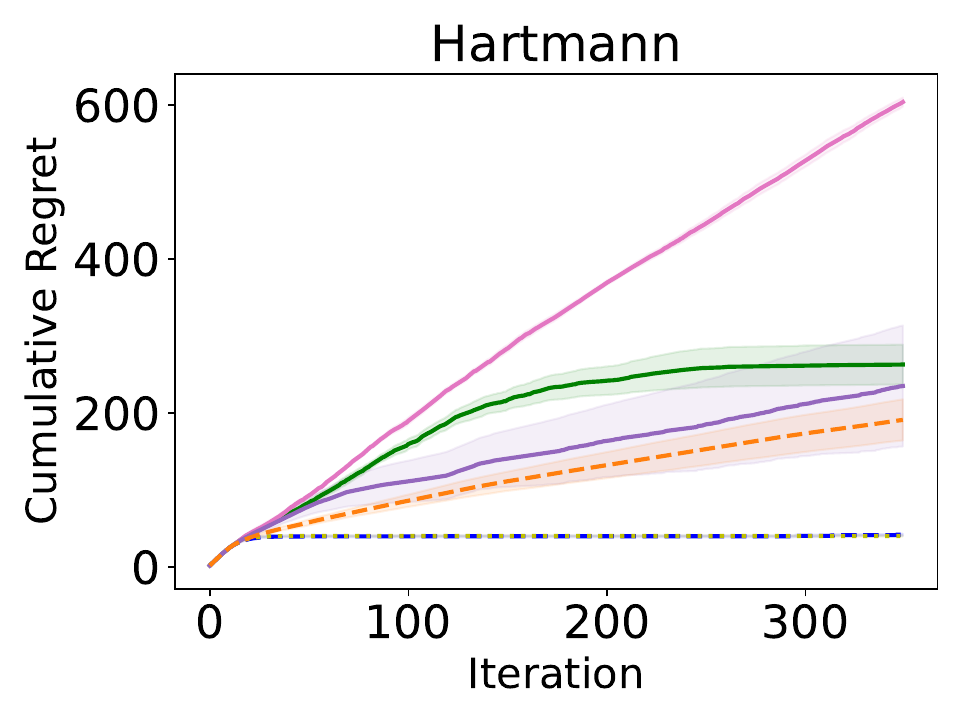}\includegraphics[width=0.24\linewidth]{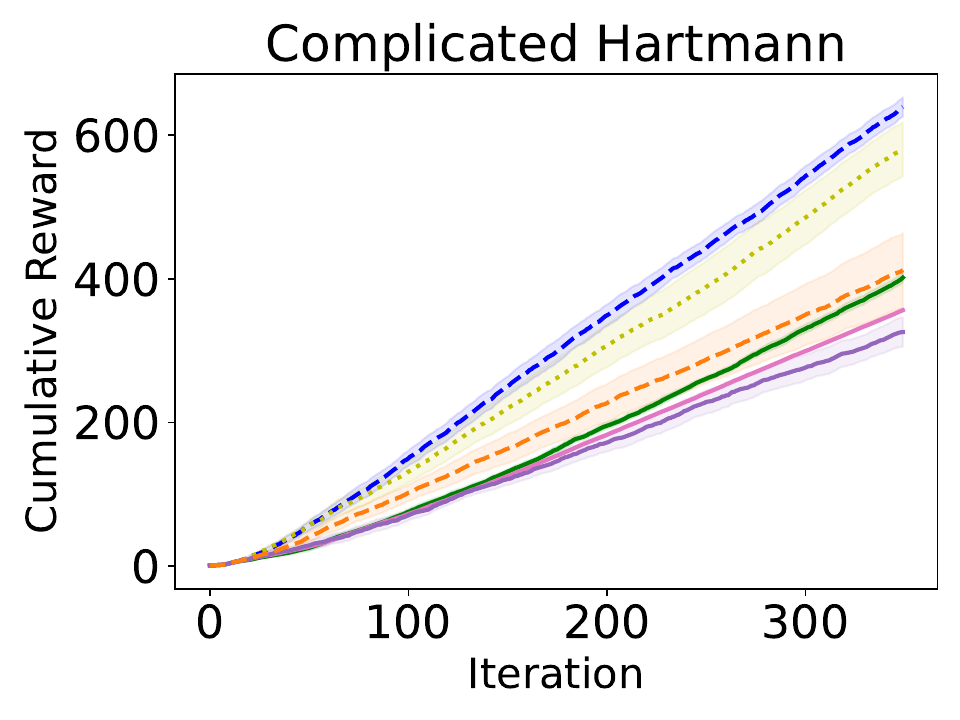}
\vspace{-1.3em}
\caption*{(a) Synthetic Functions}
\label{synthetic}
\end{subfigure}
\begin{subfigure}
    \centering
\includegraphics[width=0.24\linewidth]{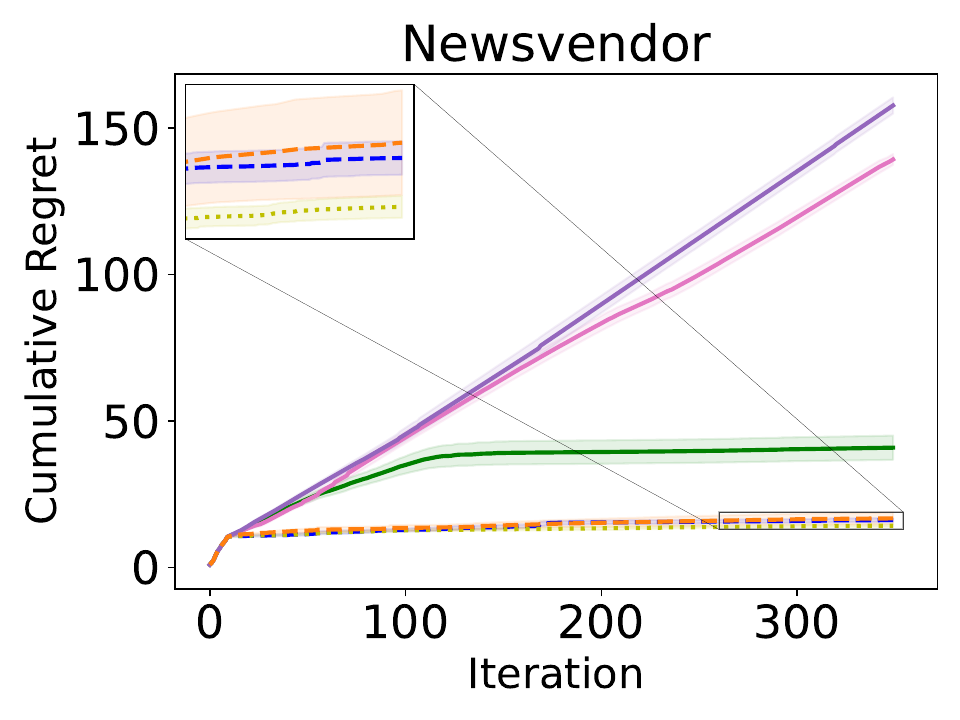}
\includegraphics[width=0.24\linewidth]{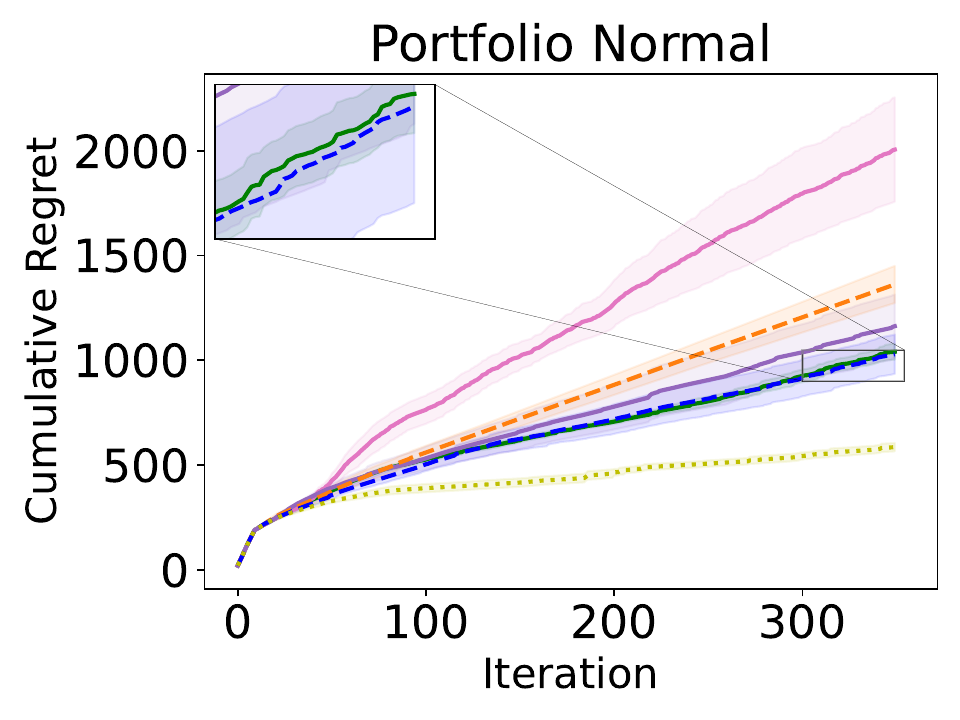}
\includegraphics[width=0.24\linewidth]{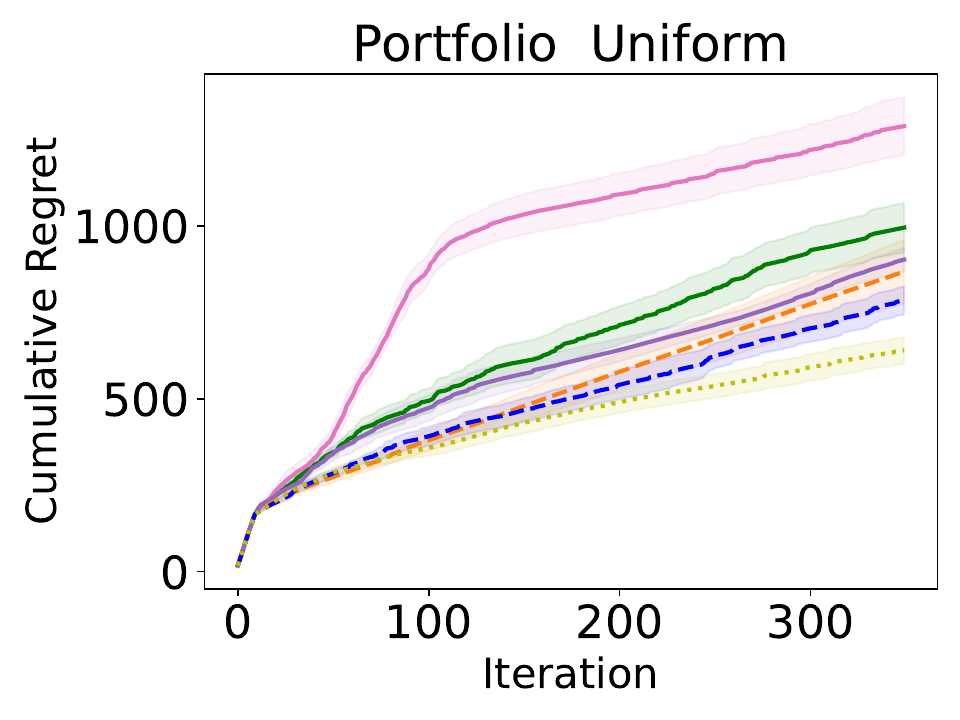}
\includegraphics[width=0.24\linewidth]{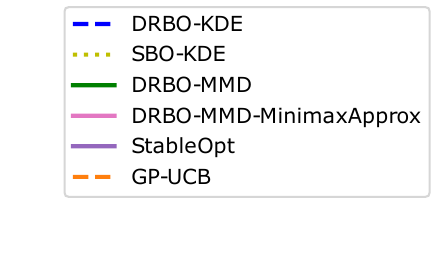}
\vspace{-1.3em}
\caption*{(b) Real-World Problems}
\label{real}
\end{subfigure}
\vspace{-0.6em}
\caption{Mean and standard error of cumulative regret (the lower the better) or cumulative reward (the higher the better).}
\vspace{-0.5em}
\end{figure*}

We conduct experiments on four commonly used synthetic test functions~\cite{synthetic_function}, in which we follow the approach of setting some dimensions as context variable from~\cite{modified_branin,risk1}. The functions include the Ackley function with one dimension set as a context variable following a normal distribution, the Modified Branin function with two dimensions set as context variable following a normal distribution, the Hartmann function with one dimension set as a context variable following a normal distribution, and the Complicated Hartmann function whose context variable follows a more complicated distribution (a mixture of six normal and two Cauchy distributions). More details can be found in Appendix~\ref{problem_definition_appendix}. For the first three functions, we use the cumulative regret in Eq.~\eqref{stochastic_cumulative_regret} as the metric,\footnote{The optimal solution is approximately obtained by optimizing the average of QMC samples using multi-restart L-BFGS.} and we calculate the expectation $\mathbb{E}_{\bm{c}\sim p(\bm{c})}[f(\bm{x},\bm{c})]$ by averaging $2^{21}$ quasi-Monte Carlo (QMC) samples.\footnote{We set it arbitrarily. Using more QMC samples leads to more accurate estimation. Due to limitations in computing resources, we use $2^{21}$ QMC samples, which, however, can guarantee the estimation accurate enough.} For the last function, due to the complexity of the context distribution, it is difficult to perform QMC sampling, so we only report the observed cumulative reward, that is $\sum_{t=1}^Tf(\bm{x}_t,\bm{c}_t)$.

The results are shown in Figure~1(a). We can observe that the proposed algorithms, SBO-KDE and DRBO-KDE, outperform all the other baselines on the synthetic functions. In the Ackley function, SBO-KDE performs slightly better than DRBO-KDE, which is because DRBO-KDE is more conservative. However, in the Complicated Hartmann function with a more complex context distribution, DRBO-KDE performs better. This is because KDE may suffer from a higher estimation error for complicated distributions, while DRBO-KDE takes into account the discrepancy between the estimated distribution and the true distribution, thus exhibiting a more robust performance. To observe the higher estimation error between the KDE and the true context distribution under the complicated distribution, we report the discrepancy measured by total variation between PDF estimated by KDE and the true context distribution under the two distributions on Hartmann function in Appendix~\ref{AppendixE}.
In the Modified Branin and original Hartmann functions, the performance of SBO-KDE and DRBO-KDE is very close. It is interesting to note that GP-UCB, which does not consider context variable, has an acceptable performance. This may be because GP-UCB can model the impact of context variable on function evaluations as evaluation noise. DRBO-MMD performs well in the Modified Branin function but performs worse in the other functions, which is related to the quality of the discretization space approximation. StableOpt performs poorly because the robust objective is too pessimistic. DRBO-MMD-MinimaxApprox also has poor performance, which could be due to the minimax approximation error being too large for these problems.

In addition to the performance of optimization, we also provide the computational complexity comparison of the algorithms in Appendix~\ref{ComplexityAppendix}. Besides, although we use a small dimension of context variable by following the experiments in DRBO literature,
where the dimension of context variable tends to be relatively
low (at most three)~\cite{drbo,drbo_wcs}, we conduct experiments on one more problem with four-dimensional context variable in Appendix~\ref{AppendixD}, to show the performance of algorithms on problems with higher dimensional context variable. 

\subsection{Real-World Problems}
We further examine the effectiveness of our algorithms on two real-world optimization tasks, including newsvendor problem and portfolio optimization. Newsvendor problem is a classic inventory management problem in stochastic optimization, where a seller pre-determines the inventory to satisfy customer demand, and portfolio optimization is the process of adjusting trading strategies to maximize the returns on investment. 

\textbf{Newsvendor problem}. The first real-world problem we consider is a continuous newsvendor problem~\cite{simoptgithub}, where a vendor purchases a certain amount of liquid denoted by $x$, and sells to customers with a demand of $c$ units. The vendor incurs a cost of $s_0$ per unit for the initial inventory and sells the liquid to customers at a price of $s_1$ per unit. Any unsold liquid at the end of the day can be sold for a salvage value of $w$ per unit. The decision variable is the purchase quantity $x$, and the context variable is the customer demand $c$. The goal is to maximize the vendor's profit, which is defined as $f(x,c) = s_1\min\{x,c\} + w\max\{0,x-c\} - s_0x$. We use the default setting of~\cite{simoptgithub}, where $s_0=5$, $s_1=9$, and $w=1$. Additionally, the customer demand $c$ follows a Burr Type XII distribution with PDF $p(c;\alpha,\beta)=\alpha\beta\frac{c^{\alpha-1}}{(1+c^{\alpha})^{\beta+1}}$, where $\alpha=2$ and $\beta=20$.

\textbf{Portfolio optimization}. The second real-world problem is portfolio optimization~\cite{risk1, risk2}. The problem involves three-dimensional decision variable (risk and trade aversion parameters, and holding cost multiplier), and two-dimensional context variable (bid-ask spread and borrowing cost). The objective function is the posterior mean of a GP trained on $3,000$ samples, which are generated by~\cite{risk1} from the CVXPortfolio problem~\cite{portfolio}. We define two tasks by setting the distribution of the context variable as normal distribution or uniform distribution. For a more detailed description, please refer to Appendix~\ref{problem_definition_appendix}. We modify the problem to the setting where the distribution is unknown and the context can be observed after each evaluation. 

The results are shown in Figure~1(b). Newsvendor problem uses $2^{21}$ QMC samples for calculating the expectation, while portfolio optimization uses $2^{16}$ QMC samples because the evaluation is more time-consuming. SBO-KDE and DRBO-KDE still outperform all the other baselines, with SBO-KDE demonstrating better performance. Among
the methods that consider context variable in the newsvendor
problem, only SBO-KDE and DRBO-KDE outperform
GP-UCB. For portfolio optimization with a normal context distribution, DRBO-MMD is competitive with DRBO-KDE, which is because the expectation under the discretized context space has a good approximation over the continuous context space. 

\section{Conclusion}
\label{conclusion}

In this paper, we consider the stochastic optimization problem with an unknown continuous context distribution, and propose the two algorithms, SBO-KDE and DRBO-KDE. The former directly optimizes the SO objective using the estimated density from KDE. The latter optimizes the distributionally robust objective considering the discrepancy between the true and estimated PDF, which is more suitable when the KDE approximation error might be high due to the complexity of the true distribution. We prove sub-linear Bayesian cumulative regret bounds for both algorithms. Furthermore, we conduct numerical experiments on synthetic functions and two real-world problems to empirically demonstrate the effectiveness of the proposed algorithms. One limitation of this work is that we assume that the distribution of context variable remains static over time. We will investigate scenarios where distributional shifts occur in our future work. 

\section*{Acknowledgements}
This work was supported by the National Key R\&D Program of China (2022ZD0116600) and National Science Foundation of China (62022039, 62276124).

\bibliography{aaai24.bib}

\newpage
\appendix
\section{Detailed Proofs}
\label{theoretical_analysis_appendix}

\subsection{Convergence Rate Analysis of SAA}
\label{theorey_saa_appendix}
In this part, we are going to prove Proposition~\ref{saa_convergence_proposition}, which demonstrates the exponential convergence rate of optimization using the SAA method. First, we recall the notations used in SAA. We let $\text{ucb}_t(\bm{x},\bm{c})=\mu_t(\bm{x},\bm{c})+\sqrt{\beta_t}\sigma_t(\bm{x},\bm{c})$, and the acquisition function at iteration $t$ is denoted as $\alpha_t(\bm{x})=\mathbb{E}_{\bm{c}\sim\hat{p}_t(\bm{c})}[\text{ucb}_t(\bm{x},\bm{c})]$. The  acquisition function is approximated by $\hat{\alpha}^M_t(\bm{x})=\frac{1}{M}\sum_{i=1}^M \text{ucb}_t(\bm{x},\hat{\bm{c}}_i)$ using Monte Carlo samples $\{\hat{\bm{c}}_i\}_{i=1}^M$ from $\hat{p}_t$. For the proof, we require the following lemma, which comes from Theorem 3 and Proposition 2 in~\cite{botorch}, and also a result from~\cite{saa_convergence}.
\begin{lemma}
   Suppose that (1) $\{\hat{\bm{c}}_i\}_{i=1}^M$ are i.i.d., (2) $ \hat{\alpha}^M_t(\bm{x})\overset{\text{a.s.}}{\rightarrow}\alpha_t(\bm{x}),\forall\bm{x}\in\mathcal{X}$, (3) there exists an integrable function $l(\bm{c})$ such that $|\textnormal{ucb}_t(\bm{x},\bm{c})-\textnormal{ucb}_t(\bm{y},\bm{c})|\leq l(\bm{c})\|\bm{x}-\bm{y}\|_1,\forall \bm{x},\bm{y}\in\mathcal{X}$, (4) $\forall\bm{x}\in\mathcal{X}$, the moment generating function $M_{\bm{x}}^{\textnormal{ucb}_t}(s):=\mathbb{E}_{\bm{c}\sim\hat{p}_t(\bm{c})}\left[e^{s\cdot \textnormal{ucb}_t(\bm{x},\bm{c})}\right]$ is finite in an open neighbourhood of $s=0$, and (5) the moment generating function $M^l(s):=\mathbb{E}_{\bm{c}\sim\hat{p}_t(\bm{c})}\left[e^{sl(\bm{c})} \right]$ is finite in an open neighbourhood of $s=0$. Then, $\forall \delta >0$, there exist $Q<\infty$ and $\eta>0$ such that $\mathbb{P}(\text{dist}(\hat{\bm{x}}_M^*, \mathcal{X}^*_t)>\delta)\leq Qe^{-\eta M}$ for all $ M\geq 1$, where dist$(\hat{\bm{x}}_M^*, \mathcal{X}^*_t)=\inf_{\bm x\in \mathcal{X}^*_t}\|\bm x- \hat{\bm{x}}_M^*\|_2$, $\hat{\bm{x}}_M^* \in \arg\max_{\bm{x}\in\mathcal{X}}\hat{\alpha}_t^M(\bm{x})$ and $\mathcal{X}^*_t=\arg\max_{\bm x\in\mathcal{X}}\alpha_t(\bm{x})$. 
\label{saa_convergence_lemma}
\vspace{-1em}
\end{lemma}
With Lemma~\ref{saa_convergence_lemma}, all that remains to be done is to verify the conditions of the lemma to prove the convergence rate of SAA, which is restated in Proposition~\ref{saa_convergence_proposition_appendix} for clearness.
\begin{proposition}
Suppose that (\romannumeral 1) $\{\hat{\bm{c}}_i\}_{i=1}^M$ are i.i.d., and (\romannumeral 2) $f$ is a GP with continuously differentiable prior mean and covariance. Then, $\forall \delta >0$, there exist $Q<\infty$ and $\eta>0$ such that $\mathbb{P}(\text{dist}(\hat{\bm{x}}_M^*, \mathcal{X}^*_t)>\delta)\leq Qe^{-\eta M}$ for all $M\geq 1$, where dist$(\hat{\bm{x}}_M^*, \mathcal{X}^*_t)=\inf_{\bm x\in \mathcal{X}^*_t}\|\bm x- \hat{\bm{x}}_M^*\|_2$, $\hat{\bm{x}}_M^* \in \arg\max_{\bm{x}\in\mathcal{X}}\hat{\alpha}_t^M(\bm{x})$ and $\mathcal{X}^*_t=\arg\max_{\bm x\in\mathcal{X}}\alpha_t(\bm{x})$.
\label{saa_convergence_proposition_appendix}
\vspace{-1.2em}
\end{proposition}
\begin{proof}
    We are to verify the five conditions in Lemma~\ref{saa_convergence_lemma}. Condition~(1) is the condition of the proposition, and condition~(2) directly follows from the strong law of large numbers. 
    For condition~(3), it is easy to verify that $\mu_t(\bm{x},\bm{c})$ and $\sigma_t(\bm{x},\bm{c})$ are continuously differentiable if $f$ is a GP with continuously differentiable prior mean and covariance \cite{botorch}. So, $\mu_t(\bm{x},\bm{c})$ and $\sigma_t(\bm{x},\bm{c})$ are Lipschitz. Let $L_{\mu_t}$ and $L_{\sigma_t}$ be the Lipschitz constants for $\mu_t$ and $\sigma_t$, respectively. Then, $|\text{ucb}_t(\bm{x},\bm{c})-\text{ucb}_t(\bm{y},\bm{c})|\leq(L_{\mu_t}+\sqrt{\beta_t}L_{\sigma_t})\|\bm x-\bm y\|_1$. So we can let $l(\bm{c})\equiv L_{\mu_t}+\sqrt{\beta_t}L_{\sigma_t}$ for condition~(3). This also verifies condition~(5), because $\mathbb{E}_{\bm{c}\sim\hat{p}_t(\bm{c})}\left[e^{sl(\bm{c})}\right]= e^{s(L_{\mu_t}+\sqrt{\beta_t}L_{\sigma_t})}$ which is finite at an open neighbourhood of $s=0$. For condition~(4), $|\text{ucb}_t(\bm{x},\bm{c})|$ is bounded by some constant $B$ because it is a continuously differentiable function on the compact set $\mathcal{X}\times\mathcal{C}$. Therefore, $\mathbb{E}_{\bm{c}\sim\hat{p}_t(\bm{c})}\left[ e^{s\cdot \text{ucb}_t(\bm{x},\bm{c})}\right]\leq e^{sB}$ and it is finite at an open neighbourhood of $s=0$. Now, conditions~(1) to~(5) in Lemma~\ref{saa_convergence_lemma} are all satisfied, implying the exponential convergence rate of SAA.
\end{proof}

\subsection{Proof of the Regret Bound of SBO-KDE}
\label{theorey_kdesbo_appendix}
In this part, we are going to prove Theorem~\ref{main_theorem_1}, which bounds the Bayesian cumulative regret (BCR) of SBO-KDE defined in Eq.~\eqref{BCR}. First, we decompose the BCR into the following three terms.
\begin{align}
        &\text{BCR}(T)
        \nonumber
    \\
    =&\mathbb{E}\left[\sum_{t=1}^T\left( \mathbb{E}_{\bm{c}\sim p}[f(\bm{x}^*,\bm{c})]-\mathbb{E}_{\bm{c}\sim p}[f(\bm{x_t},\bm{c})]  \right) \right]
\nonumber
\\
=&\mathbb{E}\left[ \sum_{t=1}^T\left( \mathbb{E}_{\bm{c}\sim p}[f(\bm{x}^*,\bm{c})]  - \mathbb{E}_{\bm{c}\sim \hat{p}_t}[f(\bm{x}^*,\bm{c})]    \right)    \right]
\nonumber
\\
& +\mathbb{E}\left[ \sum_{t=1}^T\left( \mathbb{E}_{\bm{c}\sim \hat{p}_t}[f(\bm{x}^*,\bm{c})]  - \mathbb{E}_{\bm{c}\sim \hat{p}_t}[f(\bm{x}_t,\bm{c})]    \right)    \right] 
\nonumber
\\
& +\mathbb{E}\left[ \sum_{t=1}^T\left( \mathbb{E}_{\bm{c}\sim \hat{p}_t}[f(\bm{x}_t,\bm{c})]  - \mathbb{E}_{\bm{c}\sim p}[f(\bm{x}_t,\bm{c})]    \right)    \right]. 
\label{KDESBO_BCR_decomposition}
\end{align}

The first and third terms can be considered as the error from KDE approximation, and the second term can be considered as the regret from GP-UCB. Next, we are going to derive an upper bound for each term.

First, we are going to analyze the upper bound on the first and third terms of Eq.~\eqref{KDESBO_BCR_decomposition}. Before that, we bound the expectation of the Lipschitz coefficient in the following lemma. Note that $D=D_x+D_c$ is the number of dimensions of the joint space $\mathcal{X} \times \mathcal{C} \subset [0,r]^D$ of decision and context variable. 

\begin{lemma}
    For $f$ satisfying Assumption \ref{lipschitz}, let $L_{\text{max}} = \sup_{i \in [D]}\sup_{z\in\mathcal{Z}}|\frac{\partial f(\bm{z})}{\partial z_i}|$, where $[D]=\{1,2,\dots,D\}$. Then 
    \begin{equation}
    \mathbb{E}[L_{\text{max}}]\leq {ab\sqrt{\pi}}/{2}.        
    \nonumber
    \end{equation}
    \vspace{-1.2em}
    \label{L_max_lemma}
\end{lemma}
\begin{proof}
    $\mathbb{E}[L_{\text{max}}] = \int_{0}^{\infty}\mathbb{P}(L_{\text{max}}\geq l)\, dl \leq \int_{0}^{\infty}ae^{-(l/b)^2}\, dl = {ab\sqrt{\pi}}/{2}$, where the inequality holds by Assumption~\ref{lipschitz}.
\end{proof}

A direct result from Lemma~\ref{L_max_lemma} is the bound of $\max_{\bm{c}\in\mathcal{C}} |f(\bm{x},\bm{c})|$ in expectation for any fixed $\bm{x}$.

\begin{lemma} For $f$ satisfying Assumption \ref{lipschitz} and $\bm{x}\in\mathcal{X}$, we have 
\begin{align}    \mathbb{E}\left[\max\nolimits_{\bm{c}\in\mathcal{C}}|f(\bm{x},\bm{c})|\right]\leq abDr\sqrt{\pi}/2+\sqrt{2/\pi}.
\nonumber
\end{align}
\vspace{-1.2em}
\label{bound_l_infinite}
\end{lemma}
\begin{proof}
     For a given $\bm{x}\in\mathcal{X}$, let $\bm{c}^*:=\arg\max_{\bm{c}\in\mathcal{C}}|f(\bm{x},\bm{c})|$. We choose $\bm{x}_0 \in \mathcal{X}, \bm{c}_0 \in \mathcal{C}$ and keep them fixed. Then we can write $\mathbb{E}\left[\max_{\bm{c}\in\mathcal{C}}|f(\bm{x},\bm{c})|\right]$ as
\begin{align}    
&\mathbb{E}[\max\nolimits_{\bm{c}\in\mathcal{C}}|f(\bm{x},\bm{c})|] = \mathbb{E}[|f(\bm{x},\bm{c}^*)|]
    \nonumber
\\
=&\mathbb{E}\left[|f(\bm{x},\bm{c}^*)|- |f(\bm{x}_0,\bm{c}_0)| + |f(\bm{x}_0,\bm{c}_0)| \right]
\nonumber
\\
\leq&\mathbb{E}\left[|f(\bm{x},\bm{c}^*)- f(\bm{x}_0,\bm{c}_0)| \right]+ \mathbb{E}\left[|f(\bm{x}_0,\bm{c}_0)| \right].
\nonumber
\end{align}
By the bounded Lipschitz coefficient in expectation in Lemma \ref{L_max_lemma}, we have 
\begin{equation}    \mathbb{E}\left[|f(\bm{x},\bm{c}^*)- f(\bm{x}_0,\bm{c}_0)| \right]\leq \mathbb{E}\left[L_{\text{max}}Dr\right]\leq abDr\sqrt{\pi}/2.
\nonumber
\end{equation}

For fixed $(\bm{x}_0,\bm{c}_0)$, $f(\bm{x}_0,\bm{c}_0) \sim \mathcal{N}(\bm{0}, k\left((\bm{x}_0,\bm{c}_0), (\bm{x}_0,\bm{c}_0)\right))$. So $|f(\bm{x}_0,\bm{c}_0)|$ follows a half-normal distribution with expectation  $\mathbb{E}\left[ |f(\bm{x}_0,\bm{c}_0)|
  \right]= \sqrt{{2k\left((\bm{x}_0,\bm{c}_0), (\bm{x}_0,\bm{c}_0)\right)}/{\pi}}\leq \sqrt{2/\pi}$. Thus, $\mathbb{E}\left[\max_{\bm{c}\in\mathcal{C}}|f(\bm{x},\bm{c})|\right]\leq abDr\sqrt{\pi}/2+\sqrt{2/\pi}.$
\end{proof}

With the bound of $\mathbb{E}\left[\max_{\bm{c}\in\mathcal{C}}|f(\bm{x},\bm{c})|\right]$ in Lemma~\ref{bound_l_infinite} and the error bound from KDE estimation in Lemma~\ref{l2_bound_lemma}, we can derive Lemma~\ref{bound_for_1_3_parts_kdesbo_lemma}, which will directly lead to an upper bound on the first and third terms of Eq.~\eqref{KDESBO_BCR_decomposition}. We note here that for the sake of simplicity in our analysis, we assume that $t$ contexts are observed instead of $t-1$ contexts at iteration $t$.

\begin{lemma}
    With the PDF $p$ satisfying the condition in Lemma~\ref{l2_bound_lemma}, $h_t^{(i)}=\Theta\left(t^{-1/(4+D_c)}\right)\forall i\in [D_c]$, and $f$ satisfying Assumption~\ref{lipschitz}, at iteration $t$, $\forall \bm{x} \in \mathcal{X}$, there exists a constant $\tilde{C}_1>0$, such that the following holds
    \begin{equation}
|\mathbb{E}\left[ \mathbb{E}_{\bm{c}\sim p}[f(\bm{x},\bm{c})]  - \mathbb{E}_{\bm{c}\sim \hat{p}_t}[f(\bm{x},\bm{c})]  \right]|\leq \tilde{C}_1t^{-2/(4+D_c)}.
\nonumber
    \end{equation}    \label{bound_for_1_3_parts_kdesbo_lemma}
    \vspace{-1.2em}
\end{lemma}
\begin{proof}
We can see that $\forall \bm{x} \in \mathcal{X}$,
\begin{align}
 &\left|\mathbb{E}\left[ \mathbb{E}_{\bm{c}\sim p}[f(\bm{x},\bm{c})]  - \mathbb{E}_{\bm{c}\sim \hat{p}_t}[f(\bm{x},\bm{c})]  \right] \right|
 \nonumber
 \\
 &= \left|\mathbb{E} \left[ \int_{\mathcal{C}}(p(\bm{c})-\hat{p}_t(\bm{c}))f(\bm{x},\bm{c}) \, d\bm{c} \right] \right|
 \nonumber
 \\
 &\leq \mathbb{E} \left[ \int_{\mathcal{C}}| p(\bm{c})-\hat{p}_t(\bm{c})|\cdot | f(\bm{x},\bm{c})|  \, d\bm{c} \right].
 \nonumber
\end{align}
By Hölder inequality, we have 
\begin{align}
    &\mathbb{E} \left[ \int_{\mathcal{C}}|p(\bm{c})-\hat{p}_t(\bm{c})| \cdot |f(\bm{x},\bm{c})|  \, d\bm{c} \right]
    \nonumber
    \\
    \leq &\mathbb{E} \left[ \|p(\bm{c})-\hat{p}_t(\bm{c})\|_1 \cdot \|f(\bm{x},\bm{c})\|_\infty  \right]
    \nonumber
    \\
    =&\mathbb{E}[\|p(\bm{c})-\hat{p}_t(\bm{c})\|_1]
    \mathbb{E}[\|f(\bm{x},\bm{c})\|_\infty],
    \nonumber
\end{align}
where the equality follows from the fact that the distribution of context variable and the GP are independent.

Again using Hölder inequality, we have $\mathbb{E}\left[\|p(\bm{c})-\hat{p}_t(\bm{c})\|_1\right]= \mathbb{E}\left[\int_{\mathcal{C}}|p(\bm{c})-\hat{p}_t(\bm{c})|\cdot 1\, d\bm{c}\right] \leq \mathbb{E}\left[\sqrt{\int_{\mathcal{C}}|p(\bm{c})-\hat{p}_t(\bm{c})|^2\, d\bm{c}}\sqrt{\int_{\mathcal{C}}1 \,d\bm{c}}\right] =\mathbb{E}\left[\|p(\bm{c})-\hat{p}_t(\bm{c})\|_2 r^{D_c/2}\right]$. By Lemma \ref{l2_bound_lemma}, there exists a constant $C_0>0$, such that $\mathbb{E}\left[\|p(\bm{c})-\hat{p}_t(\bm{c})\|_2\right] \leq C_0t^{-2/(4+D_c)}$ by setting $h_t^{(i)}=\Theta(t^{-1/(4+D_c)})\,\forall i \in [D_c]$. Therefore, $\mathbb{E}\left[\|p(\bm{c})-\hat{p}_t(\bm{c})\|_1\right] \leq \mathbb{E}\left[\|p(\bm{c})-\hat{p}_t(\bm{c})\|_2 r^{D_c/2}\right] \leq C_0r^{D_c/2}t^{-2/(4+D_c)}$. With Lemma~\ref{bound_l_infinite}, we have $\mathbb{E}[\|f(\bm{x},\bm{c})\|_\infty]\leq abDr\sqrt{\pi}/2+\sqrt{2/\pi}$. Let $\tilde{C}_1 = (abDr\sqrt{\pi}/2+\sqrt{2/\pi})C_0r^{D_c/2}$, and then the lemma holds.
\end{proof}
Now, the first and third terms of Eq.~\eqref{KDESBO_BCR_decomposition} can be bounded using Lemma~\ref{bound_for_1_3_parts_kdesbo_lemma}. That is, 
 \begin{align}
     &\mathbb{E}\left[ \sum_{t=1}^T\left( \mathbb{E}_{\bm{c}\sim p}[f(\bm{x}^*,\bm{c})]  - \mathbb{E}_{\bm{c}\sim \hat{p}_t}[f(\bm{x}^*,\bm{c})]    \right)    \right]
     \nonumber
     \\
     &+\mathbb{E}\left[ \sum_{t=1}^T\left( \mathbb{E}_{\bm{c}\sim \hat{p}_t}[f(\bm{x}_t,\bm{c})]  - \mathbb{E}_{\bm{c}\sim p}[f(\bm{x}_t,\bm{c})]    \right)    \right]
     \nonumber
    \\
     \leq &2\sum_{t=1}^T \tilde{C}_1t^{-2/(4+D_c)} \leq 2C_1T^{(2+D_c)/(4+D_c)},
     \label{bound_for_1_3_parts_kdesbo}
 \end{align}
where $C_1=\tilde{C}_1(4+D_c)/(2+D_c)$. Next, we are going to bound the second term of Eq.~\eqref{KDESBO_BCR_decomposition}, which mainly comes from the regret of GP-UCB:
\begin{equation}
  \mathbb{E}\left[ \sum_{t=1}^T\left( \mathbb{E}_{\bm{c}\sim \hat{p}_t}[f(\bm{x}^*,\bm{c})]  - \mathbb{E}_{\bm{c}\sim \hat{p}_t}[f(\bm{x}_t,\bm{c})]    \right)    \right].
  \label{regret_from_gpucb_kdesbo}
\end{equation}
With the similar idea in~\cite{BCR_TS} and~\cite{gpucb}, we discretize the decision space $\mathcal{X}$ at iteration $t$ into $\tilde{\mathcal{X}}_t$, where $|\tilde{\mathcal{X}}_t|=(\tau_t)^{D_x}$, which means that we divide each coordinate of $\mathcal{X}$ into $\tau_t$ parts equally. We set $\tau_t={t^2D_xabr\sqrt{\pi}/2}$. Denote the closest point to $\bm{x}$ in $\tilde{\mathcal{X}}_t$ as $[\bm{x}]_t$. We decompose the regret Eq.~\eqref{regret_from_gpucb_kdesbo} into three terms.
    \begin{align}    
&\mathbb{E}\left[ \sum_{t=1}^T\left( \mathbb{E}_{\bm{c}\sim \hat{p}_t}[f(\bm{x}^*,\bm{c})]  - \mathbb{E}_{\bm{c}\sim \hat{p}_t}[f(\bm{x}_t,\bm{c})]    \right)    \right]
\nonumber
\\
  \leq &\mathbb{E}\left[ \sum_{t=1}^T\left( \mathbb{E}_{\bm{c}\sim \hat{p}_t}[f([\bm{x}^*]_t,\bm{c})]- \mathbb{E}_{\bm{c}\sim \hat{p}_t}[\text{ucb}_t([\bm{x}^*]_t,\bm{c})] \right)\right]
\nonumber
  \\
&+ \mathbb{E}\left[ \sum_{t=1}^T\left( \mathbb{E}_{\bm{c}\sim \hat{p}_t}[\text{ucb}_t(\bm{x}_t,\bm{c})]- \mathbb{E}_{\bm{c}\sim \hat{p}_t}[f(\bm{x}_t,\bm{c})] \right)\right]
\nonumber
\\
& +\mathbb{E}\left[ \sum_{t=1}^T\left( \mathbb{E}_{\bm{c}\sim \hat{p}_t}[f(\bm{x}^*,\bm{c})]- \mathbb{E}_{\bm{c}\sim \hat{p}_t}[f([\bm{x}^*]_t, \bm{c})] \right)\right],
  \label{gp-ucb-regret-decomposition-kdesbo}
\end{align}
where $\text{ucb}_t(\bm{x})= \mu_t(\bm{x}) +\sqrt{\beta_t}\sigma_t(\bm{x})$, and the inequality follows directly from $ \mathbb{E}_{\bm{c}\sim \hat{p}_t}[\text{ucb}_t(\bm{x}_t,\bm{c})] \geq \mathbb{E}_{\bm{c}\sim \hat{p}_t}[\text{ucb}_t([\bm{x}^*]_t,\bm{c})]$ by $\bm{x}_t \in \arg\max_{\bm{x}\in\mathcal{X}}\mathbb{E}_{\bm{c}\sim \hat{p}_t}[\text{ucb}_t(\bm{x},\bm{c})]$.

The third term of Eq.~\eqref{gp-ucb-regret-decomposition-kdesbo} can be directly bounded as follows:
\begin{align}
& \mathbb{E}\left[ \sum_{t=1}^T\left( \mathbb{E}_{\bm{c}\sim \hat{p}_t}[f(\bm{x}^*,\bm{c})]- \mathbb{E}_{\bm{c}\sim \hat{p}_t}[f([\bm{x}^*]_t, \bm{c})] \right)\right]
\nonumber
\\
=&\mathbb{E}\left[ \sum_{t=1}^T \mathbb{E}_{\bm{c}\sim \hat{p}_t}\left[f(\bm{x}^*,\bm{c})-f([\bm{x}^*]_t, \bm{c}) \right] \right]
\nonumber
\\
\leq &\mathbb{E}\left[ \sum_{t=1}^T \mathbb{E}_{\bm{c}\sim \hat{p}_t}\left[|f(\bm{x}^*,\bm{c})-f([\bm{x}^*]_t, \bm{c})| \right] \right]
\nonumber
\\
\leq &\mathbb{E}\left[ \sum_{t=1}^T \mathbb{E}_{\bm{c}\sim \hat{p}_t}\left[L_\text{max}D_xr/\tau_t \right] \right]
\nonumber
\\
=&\sum_{t=1}^T\mathbb{E}\left[L_\text{max}\right]D_xr/\tau_t
\nonumber
\\
\leq&\sum_{t=1}^T\frac{ab\sqrt{\pi}D_xr}{2}*\frac{2}{t^2D_xabr\sqrt{\pi}}
\nonumber
\\
=&\sum_{t=1}^T\frac{1}{t^2}\leq\frac{\pi^2}{6},
\label{third_part_gpucb_regret_kdesbo}
\end{align}
where the second inequality follows from $|f(\bm x ^*, \bm c) - f([\bm x ^*]_t, \bm c)|\leq L_{\text{max}}\|(\bm x^*,\bm c)-([\bm x^*]_t,\bm c)\|_1\leq L_{\text{max}}D_x(r/\tau_t)$, and the third inequality follows from the result of Lemma~\ref{L_max_lemma} that $\mathbb{E}[L_{\text{max}}]\leq ab\sqrt{\pi}/2$ as well as $\tau_t={t^2D_xabr\sqrt{\pi}/2}$.

For the first term of Eq.~\eqref{gp-ucb-regret-decomposition-kdesbo}, we need the following bound for points in the discretized space.

\begin{lemma}
Let $\beta_t=2\log{\frac{t^2(\tau_t)^{D_x}}{\sqrt{2\pi}}}$ with $\tau_t={t^2D_xabr\sqrt{\pi}/2}$. At iteration $t$, for $\textnormal{ucb}_t(\bm{x},\bm c)= \mu_t(\bm{x}, \bm c) +\sqrt{\beta_t}\sigma_t(\bm{x},\bm c)$, for all $\bm{x}\in\tilde{\mathcal{X}_t}$, the following holds
\begin{equation}
\mathbb{E}\left[\left(f(\bm{x},\bm{c})- \textnormal{ucb}_t(\bm{x},\bm{c}) \right)_+\right] \leq\frac{1}{t^2(\tau_t)^{D_x}},
\nonumber
\end{equation}
where $(X)_+:=\max \{0,X\}$. 
\label{lemma_for_disretize}
\end{lemma}
\begin{proof}
Using the tower property of conditional expectation, we have 
\begin{align}
&f(\bm{x},\bm{c})- \text{ucb}_t(\bm{x},\bm{c}) \mid \mathcal{D}_{t-1} 
\nonumber
\\
\sim &\mathcal{N}\left(-\sqrt{\beta}_t\sigma_t(\bm{x},\bm{c}),\sigma_t^2(\bm{x},\bm{c})\right).
    \nonumber
\end{align}
Using the fact that for $X \sim \mathcal{N}(\mu,\sigma^2)$ with $\mu\leq0$, $\mathbb{E}[(X)_+]=\frac{\sigma}{\sqrt{2\pi}}e^{-\mu^2/(2\sigma^2)}$, we have 
\begin{align}
&\mathbb{E}\left[\left(f(\bm{x},\bm{c})- \text{ucb}_t(\bm{x},\bm{c})\right)_+ \mid \mathcal{D}_{t-1} \right] 
\nonumber
\\
= &\frac{\sigma_t(\bm{x}, \bm{c})}{\sqrt{2\pi}}e^{-\beta_t/2}\leq \frac{1}{t^2(\tau_t)^{D_x}}.
    \nonumber
\end{align}
Thus, the lemma holds.
\end{proof}

Then, we can use Lemma~\ref{lemma_for_disretize} to derive an upper bound on the first term of Eq.~\eqref{gp-ucb-regret-decomposition-kdesbo} for the optimal point $[\bm{x}^*]_t$ in the discretized space.

\begin{lemma}
    Let $\beta_t=2\log{\frac{t^2(\tau_t)^{D_x}}{\sqrt{2\pi}}}$ with $\tau_t={t^2D_xabr\sqrt{\pi}/2}$. At iteration $t$, for $\textnormal{ucb}_t(\bm{x}, \bm c)= \mu_t(\bm{x}, \bm c) +\sqrt{\beta_t}\sigma_t(\bm{x}, \bm c)$, the following holds
    \begin{align}
        &\mathbb{E}\left[ \sum_{t=1}^T\left( \mathbb{E}_{\bm{c}\sim \hat{p}_t}[f([\bm{x}^*]_t,\bm{c})]- \mathbb{E}_{\bm{c}\sim \hat{p}_t}[\textnormal{ucb}_t([\bm{x}^*]_t,\bm{c})] \right)\right] 
        \nonumber
        \\
        \leq &\frac{\pi^2}{6}.
\label{bound_first_part_gpucb_regret_kdesbo_ineq}    
\end{align} \label{bound_first_part_gpucb_regret_kdesbo}
\vspace{-1em}   
\end{lemma}
\begin{proof}
    At iteration $t$, we have 
\begin{align}
&\mathbb{E}\left[ \mathbb{E}_{\bm{c}\sim \hat{p}_t}\left[f([\bm{x}^*]_t,\bm{c})\right]- \mathbb{E}_{\bm{c}\sim \hat{p}_t}\left[\text{ucb}_t([\bm{x}^*]_t,\bm{c})\right] \right] 
\nonumber
\\
= &\mathbb{E}\left[ \mathbb{E}_{\bm{c}\sim \hat{p}_t}\left[f([\bm{x}^*]_t,\bm{c})- \text{ucb}_t([\bm{x}^*]_t,\bm{c})\right] \right]
\nonumber
\\
\leq &\mathbb{E}\left[ \int_{\mathcal{C}}\hat{p}_t(\bm{c})\left(f([\bm{x}^*]_t,\bm{c})- \text{ucb}_t([\bm{x}^*]_t,\bm{c})\right)_+\, d\bm{c} \right]
\nonumber
\\
\leq &\sum_{\bm{x}\in \tilde{\mathcal{X}}_t}\mathbb{E}\left[\int_{\mathcal{C}}\hat{p}_t(\bm{c})\left(f(\bm{x},\bm{c})- \text{ucb}_t(\bm{x},\bm{c})\right)_+\, d\bm{c} \right],
\nonumber
\end{align}
where the first inequality follows from $f([\bm{x}^*]_t,\bm{c})- \text{ucb}_t([\bm{x}^*]_t,\bm{c}) \leq \left(f([\bm{x}^*]_t,\bm{c})- \text{ucb}_t([\bm{x}^*]_t,\bm{c})\right)_+$, and the second inequality follows from $[\bm{x}^*]_t \in \tilde{\mathcal{X}}_t$.

Using Fubini's theorem to exchange the integral and the expectation, we have
\begin{align}
    &\sum_{\bm{x}\in \tilde{\mathcal{X}}_t}\mathbb{E}\left[\int_{\mathcal{C}}\hat{p}_t(\bm{c})\left(f(\bm{x},\bm{c})- \text{ucb}_t(\bm{x},\bm{c})\right)_+\, d\bm{c}  \right]
    \nonumber
    \\
    =&\sum_{\bm{x}\in \tilde{\mathcal{X}}_t}\int_{\mathcal{C}}\mathbb{E}\left[ \hat{p}_t(\bm{c})\left(f(\bm{x},\bm{c})- \text{ucb}_t(\bm{x},\bm{c})\right)_+ \right] \, d\bm{c} 
    \nonumber
    \\
    =&\sum_{\bm{x}\in \tilde{\mathcal{X}}_t}\int_{\mathcal{C}}\mathbb{E}\left[ \hat{p}_t(\bm{c})\right]
\mathbb{E}\left[\left(f(\bm{x},\bm{c})- \text{ucb}_t(\bm{x},\bm{c})\right)_+ \right] \, d\bm{c},
    \nonumber
\end{align}
where the last equality follows from the fact that the distribution of context variable and the GP are independent.

With Lemma~\ref{lemma_for_disretize}, and using Fubini's theorem again, we have 
\begin{align}
&\sum_{\bm{x}\in \tilde{\mathcal{X}}_t}\int_{\mathcal{C}}\mathbb{E}\left[ \hat{p}_t(\bm{c})\right] \mathbb{E}\left[\left(f(\bm{x},\bm{c})- \text{ucb}_t(\bm{x},\bm{c})\right)_+ \right] \, d\bm{c} 
    \nonumber
    \\
    \leq&\sum_{\bm{x}\in \tilde{\mathcal{X}}_t}\int_{\mathcal{C}}\mathbb{E}\left[ \hat{p}_t(\bm{c})\right] \frac{1}{t^2(\tau_t)^{D_x}} \, d\bm{c}
    \nonumber
    \\
=&\sum_{\bm{x}\in \tilde{\mathcal{X}}_t}\mathbb{E}\left[ \int_{\mathcal{C}} \hat{p}_t(\bm{c}) \, d\bm{c} \right] \frac{1}{t^2(\tau_t)^{D_x}}
    =\frac{1}{t^2}.
    \nonumber
\end{align}
Then, summing $1/t^2$ from $1$ to $T$ leads to the conclusion, i.e., Eq.~\eqref{bound_first_part_gpucb_regret_kdesbo_ineq}.
\end{proof}

By now, the only term that is left to be bounded is the second term of Eq.~\eqref{gp-ucb-regret-decomposition-kdesbo}. We then derive its upper bound in Lemma~\ref{kdesbo_second_parts_gpu_ucb_regret}.

\begin{lemma}
        Let $\beta_t=2\log{\frac{t^2(\tau_t)^{D_x}}{\sqrt{2\pi}}}$ with $\tau_t={t^2D_xabr\sqrt{\pi}/2}$ and $h_t^{(i)}=\Theta\left(t^{-1/(4+D_c)}\right)\forall i\in [D_c]$. At iteration $t$, for $\textnormal{ucb}_t(\bm{x}, \bm c)= \mu_t(\bm{x}, \bm c) +\sqrt{\beta_t}\sigma_t(\bm{x}, \bm c)$, the following holds
    \begin{align}
        &\mathbb{E}\left[ \sum_{t=1}^T\left( \mathbb{E}_{\bm{c}\sim \hat{p}_t}[\textnormal{ucb}_t(\bm{x}_t,\bm{c})]- \mathbb{E}_{\bm{c}\sim \hat{p}_t}[f(\bm{x}_t,\bm{c})] \right)\right]
        \nonumber
        \\
    \leq&\sqrt{\beta_T\gamma_TC_2}(\sqrt{T^{D_c/(4+D_c)}}+\sqrt{T}).
\label{kdesbo_second_parts_gpu_ucb_regret_ineq}
    \end{align}
    where $C_2$ is a positive constant, $\gamma_T=\max_{|\mathcal{D}|=T} I(\bm{y}_\mathcal{D},  \bm{f}_\mathcal{D})$, $I(\cdot,\cdot)$ is the information gain, and $\bm{y}_\mathcal{D}, \bm{f}_\mathcal{D}$ are the noisy and true observations of a data set $\mathcal{D}$, respectively.
    \label{kdesbo_second_parts_gpu_ucb_regret}
\end{lemma}
\begin{proof}
Using the tower property of conditional expectation and Fubini's theorem, at iteration $t$, we have
\begin{align}
&\mathbb{E}\left[ \mathbb{E}_{\bm c \sim \hat{p}_t}[\text{ucb}_t(\bm x_t, \bm c)] - \mathbb{E}_{\bm c \sim \hat{p}_t}[f(\bm x_t, \bm c)] \right]
\nonumber
\\
 =&\mathbb{E}\left[  \int_{\mathcal{C}} \hat{p}_t(\bm{c})(\text{ucb}_t(\bm{x}_t,\bm{c})-f(\bm{x}_t,\bm{c}) ) \, d\bm{c}  \right]  
 \nonumber
 \\
=&\mathbb{E}\left[  \mathbb{E}\left[\int_{\mathcal{C}} \hat{p}_t(\bm{c})(\text{ucb}_t(\bm{x}_t,\bm{c})-f(\bm{x}_t,\bm{c}) )\, d\bm{c} \mid\mathcal{D}_{t-1}\right]  \right]
 \nonumber
 \\
 =&\mathbb{E}\bigg[  \int_{\mathcal{C}} \mathbb{E}\left[\hat{p}_t(\bm{c}) \mid \mathcal{D}_{t-1}\right]
\nonumber
\\
&\quad\quad\cdot\mathbb{E}\left[\text{ucb}_t(\bm{x}_t,\bm{c})-f(\bm{x}_t,\bm{c}) \mid\mathcal{D}_{t-1}\right]\, d\bm{c}\bigg]
 \nonumber
 \\
 =&\int_{\mathcal{C}}\mathbb{E}\left[ \mathbb{E}[\hat{p}_t(\bm{c})\mid \mathcal{D}_{t-1}]\mathbb{E}[\sqrt{\beta_t}\sigma_t(\bm{x}_t,\bm{c})\mid \mathcal{D}_{t-1}]  
 \right] \, d \bm{c}
 \nonumber
 \\
 =&\int_{\mathcal{C}}\mathbb{E}\left[\hat{p}_t(\bm{c})\sqrt{\beta_t}\sigma_t(\bm{x}_t,\bm{c})\right]\, d\bm{c} 
 \nonumber
 \\
=&\mathbb{E}\left[ \int_{\mathcal{C}}\hat{p}_t(\bm{c})\sqrt{\beta_t}\sigma_t(\bm{x}_t,\bm{c})\, d\bm{c} 
    \right],
    \nonumber
\end{align}
where the third equality follows from the fact that the distribution of context variable and the GP are independent. With the same process in the proof of Lemma~\ref{bound_for_1_3_parts_kdesbo_lemma}, we have $\mathbb{E}\left[\|p(\bm{c})-\hat{p}_t(\bm{c})\|_1\right]\leq C_0r^{D_c/2}t^{-2/(4+D_c)}$. Then,
\begin{align}
 &\mathbb{E}\left[ \sum_{t=1}^T\left( \mathbb{E}_{\bm{c}\sim \hat{p}_t}[\text{ucb}_t(\bm{x}_t,\bm{c})]- \mathbb{E}_{\bm{c}\sim \hat{p}_t}[f(\bm{x}_t,\bm{c})] \right)\right]
 \nonumber
 \\
 = &\sum_{t=1}^{T}\mathbb{E}\left[ \int_{\mathcal{C}}p(\bm{c})\sqrt{\beta_t}\sigma_t(\bm{x}_t,\bm{c})\, d\bm{c} 
    \right] 
    \nonumber
    \\
    &+ \sum_{t=1}^{T}\mathbb{E}\left[ \int_{\mathcal{C}}(\hat{p}_t(\bm{c})-p(\bm{c}))\sqrt{\beta_t}\sigma_t(\bm{x}_t,\bm{c})\, d\bm{c} 
    \right]
    \nonumber
    \\
    \leq &\mathbb{E}\left[\int_{\mathcal{C}}p(\bm{c}) \sqrt{\sum_{t=1}^T\beta_t\sigma_t^2(\bm{x}_t,\bm{c})}\sqrt{T} \, d\bm{c} \right] 
    \nonumber
    \\
    &+ \mathbb{E}\left[\sum_{t=1}^T \|\hat{p}_t(\bm{c})-p(\bm{c}) \|_1 \sqrt{\beta}_t \sigma_t(\bm{x}_t,\bm{c}^*_t)\right]
    \nonumber
    \\    \leq&\sqrt{\beta_T\gamma_T T C}
    \nonumber
    \\
    &+ C_0r^{D_c/2}\mathbb{E}\left[ \sum_{t=1}^Tt^{-2/(4+D_c)}\sqrt{\beta_t}\sigma_t(\bm{x}_t, \bm{c}_t^*) \right]
    \nonumber
    \\
\leq &\sqrt{\beta_T\gamma_T T C} \nonumber
\\
&+ C_0r^{D_c/2}\mathbb{E}\left[ \sqrt{\sum_{t=1}^Tt^{-4/(4+D_c)}}\sqrt{\sum_{t=1}^T \beta_t\sigma_t^2(\bm{x}_t, \bm{c}_t^*)} \right]
\nonumber
\\
 \leq &\sqrt{\beta_T\gamma_T T C}
\nonumber
\\
&+C_0r^{D_c/2}\sqrt{\beta_T\gamma_T T^{D_c/(4+D_c)} C(4+D_c)/D_c}
\nonumber
\\
\leq&\sqrt{\beta_T\gamma_TC_2}(\sqrt{T^{D_c/(4+D_c)}}+\sqrt{T}),
\nonumber
\end{align}
where $\bm c_t^*=\arg\max_{\bm{c}\in\mathcal{C}}\sigma_t(\bm{x}_t,\bm{c})$, $C=8/\log{(1+\sigma^2)}$, $\sigma$ is the standard deviation of observation noise, the first inequality follows from Cauchy inequality $\sum_{t=1}^T \sqrt{\beta_t}\sigma_t(\bm x _t,\bm c)\leq \sqrt{\sum_{t=1}^T\beta_t\sigma_t^2(\bm x_t,\bm c)}\sqrt{T}$, and Hölder inequality $\int_{\mathcal{C}} (\hat{p}_t(\bm c)-p(\bm c))\sigma_t(\bm x_t,\bm c)\, d\bm c \leq  \|\hat{p}_t(\bm c)-p(\bm c)\|_1 \sigma_t(\bm x_t,\bm c_t^*)$, the second inequality follows from Lemma 5.4 of~\cite{gpucb}, i.e., $\sum_{t=1}^T \beta_t\sigma_t^2(\bm x_t, \bm c)\leq \beta_T\gamma_TC$ and $\mathbb{E}\left[\|p(\bm{c})-\hat{p}_t(\bm{c})\|_1\right]\leq C_0r^{D_c/2}t^{-2/(4+D_c)}$, the third inequality also follows from Cauchy inequality $\sum_{t=1}^T t^{-2/(4+D_c)}\sqrt{\beta_t}\sigma_t(\bm x_t,\bm c_t^*) \leq \sqrt{\sum_{t=1}^Tt^{-4/(4+D_c)}}\sqrt{\sum_{t=1}^T \beta_t\sigma_t^2(\bm{x}_t, \bm{c}_t^*)}$, the last inequality also follows from Lemma 5.4 of~\cite{gpucb}, i.e., $\sum_{t=1}^T \beta_t\sigma_t^2(\bm{x}_t, \bm{c}_t^*) \leq \beta_T\gamma_TC$, and the last inequality holds by letting $C_2=C\max\{1,C_0^2r^{D_c}(4+D_c)/D_c\}$. Thus, the lemma holds.
\end{proof}

By now, we can directly get the upper bound on the BCR of SBO-KDE, which is re-stated in Theorem~\ref{main_theorem1_appendix} for clearness, by directly summing over all components of Eq.~\eqref{KDESBO_BCR_decomposition} that is bounded.
\begin{theorem}[Theorem~\ref{main_theorem_1} in the main paper]
    Let $ \beta_t=2\log(t^2/\sqrt{2\pi}) +2D_x\log(t^2D_xabr\sqrt{\pi}/2)$. With the underlying PDF $p(\bm{c})$ satisfying the condition in Lemma~\ref{l2_bound_lemma}, $\hat{p}_t(\bm{c})$ defined as Eq.~\eqref{KDE} and $h_t^{(i)}=\Theta\left(t^{-1/(4+D_c)}\right)\forall i\in [D_c]$. Then, the BCR of SBO-KDE satisfies
\begin{align}
\textnormal{BCR}(T)\leq&\frac{\pi^2}{3} + \sqrt{\beta_T\gamma_TC_2}\left(\sqrt{T^{D_c/(4+D_c)}}+\sqrt{T}\right)
\nonumber
\\
& + 2C_1T^{\frac{2+D_c}{4+D_c}},
\nonumber
\end{align}
where $C_1,C_2>0$ are constants, $\gamma_T=\max_{|\mathcal{D}|=T} I(\bm{y}_\mathcal{D},  \bm{f}_\mathcal{D})$, $I(\cdot,\cdot)$ is the information gain, and $\bm{y}_\mathcal{D}, \bm{f}_\mathcal{D}$ are the noisy and true observations of a data set $\mathcal{D}$, respectively.
\label{main_theorem1_appendix}
\end{theorem}
\begin{proof}
    With Eq.~\eqref{KDESBO_BCR_decomposition} and Eq.~\eqref{gp-ucb-regret-decomposition-kdesbo}, BCR$(T)$ has the following upper bound:
    \begin{align}
&\textnormal{BCR}(T)
\nonumber
\\
\leq&\mathbb{E}\left[ \sum_{t=1}^T\left( \mathbb{E}_{\bm{c}\sim p}[f(\bm{x}^*,\bm{c})]  - \mathbb{E}_{\bm{c}\sim \hat{p}_t}[f(\bm{x}^*,\bm{c})]    \right)    \right]
\nonumber
\\
&+\mathbb{E}\left[ \sum_{t=1}^T\left( \mathbb{E}_{\bm{c}\sim \hat{p}_t}[f(\bm{x}_t,\bm{c})]  - \mathbb{E}_{\bm{c}\sim p}[f(\bm{x}_t,\bm{c})]    \right)    \right] 
\nonumber
\\
\nonumber
  & +\mathbb{E}\left[ \sum_{t=1}^T\left( \mathbb{E}_{\bm{c}\sim \hat{p}_t}[f([\bm{x}^*]_t,\bm{c})]- \mathbb{E}_{\bm{c}\sim \hat{p}_t}[\text{ucb}_t([\bm{x}^*]_t,\bm{c})] \right)\right]
\nonumber
  \\
& +\mathbb{E}\left[ \sum_{t=1}^T\left( \mathbb{E}_{\bm{c}\sim \hat{p}_t}[\text{ucb}_t(\bm{x}_t,\bm{c})]- \mathbb{E}_{\bm{c}\sim \hat{p}_t}[f(\bm{x}_t,\bm{c})] \right)\right]
\nonumber
\\
& +\mathbb{E}\left[ \sum_{t=1}^T\left( \mathbb{E}_{\bm{c}\sim \hat{p}_t}[f(\bm{x}^*,\bm{c})]- \mathbb{E}_{\bm{c}\sim \hat{p}_t}[f([\bm{x}^*]_t, \bm{c})] \right)\right].
\nonumber
\end{align}
The sum of the first and second terms has an upper bound of $2C_1T^{\frac{2+D_c}{4+D_c}}$ by Eq.~\eqref{bound_for_1_3_parts_kdesbo}. The third term has a bound of $\pi^2/6$ by Eq.~\eqref{bound_first_part_gpucb_regret_kdesbo_ineq}. The fourth term has a bound of $\sqrt{\beta_T\gamma_TC_2}(\sqrt{T^{D_c/(4+D_c)}}+\sqrt{T})$ by Eq.~\eqref{kdesbo_second_parts_gpu_ucb_regret_ineq}. The fifth term has a bound of $\pi^2/6$ by Eq.~\eqref{third_part_gpucb_regret_kdesbo}. By summing all these up, we can get the result.
\end{proof}

\subsection{Theoretical Analysis of DRBO-KDE}
\label{theorey_kdedrbo_appendix}

In this part, we first give the derivation of an equivalent form of the DRO objective under total variation in Proposition~\ref{dro_transformation_prop}. Then we will prove the regret bound of DRBO-KDE in Theorem~\ref{main_theorem_2}. 

\subsubsection{Equivalence Analysis of DRO under Total Variation}

For clearness, we restate the equivalence of two problems of DRO in Proposition~\ref{dro_transformation_prop_appendix}, which equivalently transforms the DRO objective under total variation (i.e., the inner convex minimization problem, which has infinite dimensions) into a two-dimensional convex SO problem. The main derivation idea is to use Lagrange multiplier to transform the original problem into a dual problem.
\begin{proposition}[DRO under Total Variation]
    Given a bounded function $f(\bm{x},\bm{c})$ over $\mathcal{X} \times \mathcal{C}$, a radius $\delta>0$ and a PDF $p(\bm{c})$, we have
    \begin{align}
&\min_{q\in\mathcal{B}(p,\delta)}\mathbb{E}_{\bm{c}\sim q(\bm{c})}\left[ f(\bm{x},\bm{c})  \right] 
\nonumber
\\
= &\max_{(\alpha,\beta)\in S(f)}\mathbb{E}_{\bm{c}\sim p(\bm{c})}\left[-\beta-\delta\alpha+ \min\{f(\bm{x},\bm{c})+\beta, \alpha\} \right],
\nonumber
\end{align}
where $\mathcal{B}(p,\delta)=\{q: d(q,p)\leq \delta\}$, $S(f):=\{ (\alpha,\beta): \beta\in \mathbb{R}, \alpha\geq 0, \alpha+\beta\geq -\inf\nolimits_{\bm{c}\in\mathcal{C}}f(\bm{x},\bm{c}) \}.$
\label{dro_transformation_prop_appendix}
\end{proposition}
\begin{proof}
Let $Q$ and $P$ be the measure of PDFs $q$ and $p$, respectively. That is, $dQ=q$ and $dP = p$. Recall that the total variation $d_{TV}(Q,P)=d(q,p) = \int_{\mathcal{C}}p(\bm{c})\phi\left(\frac{q(\bm{c})}{p(\bm{c})}\right)\, d\bm{c}$ with $\phi(x)=|x-1|$ and $Q\ll P$. This is a special case of $\phi$-divergence. Here, we directly derive an equivalent form for DRO under general $\phi$-divergence. The DRO objective under $\phi$-divergence can be written as
\begin{align}
\min_{q}&\int_{\mathcal{C}}f(\bm{x},\bm{c})q(\bm{c})\, d\bm{c},\quad
    \nonumber
    \\
    \text{s.t. }&\int_{\mathcal{C}}p(\bm{c})\phi\left(\frac{q(\bm{c})}{p(\bm{c})}\right)\,d\bm{c}\leq \delta; 
    \nonumber
    \\&\text{ }\int_{\mathcal{C}} q(\bm{c})\,d\bm{c}=1;\text{ }q(\bm{c})\geq 0,\forall \bm{c}\in\mathcal{C}.
    \label{primal_problem}
\end{align}

By introducing Lagrange multipliers $\alpha\geq 0,\beta\in\mathbb{R}$ to the first two constraints, respectively, we can get the dual problem of Eq.~\eqref{primal_problem}:
\begin{align}
    &\max_{\alpha\geq0,\beta}\min_{q\geq 0}\text{ }\int_{\mathcal{C}}f(\bm{x},\bm{c})q(\bm{c})\, d\bm{c}+\beta\left(\int_{\mathcal{C}}q(\bm{c})\,d\bm{c}-1\right)
   \nonumber
   \\ &+\alpha\left(\int_{\mathcal{C}}p(\bm{c})\phi\left(\frac{q(\bm{c})}{p(\bm{c})}\right)\,d\bm{c}- \delta\right)
    \nonumber
    \\
    \iff &\max_{\alpha\geq0,\beta}\min_{q\geq 0}\text{ }\int_{\mathcal{C}}\bigg((f(\bm{x},\bm{c})+\beta)q(\bm{c})
    \nonumber
    \\
    &+\alpha p(\bm{c})\phi\left(\frac{q(\bm{c})}{p(\bm{c})}\right)\bigg)\, d\bm{c}-\beta-\delta\alpha
    \nonumber
    \\
    \overset{(1)}{\iff} &\max_{\alpha\geq0,\beta}\int_{\mathcal{C}}\min_{q\geq 0}\text{ }\bigg((f(\bm{x},\bm{c})+\beta)q(\bm{c})
    \nonumber
    \\&+\alpha p(\bm{c})\phi\left(\frac{q(\bm{c})}{p(\bm{c})}\right)\bigg)\, d\bm{c}-\beta-\delta\alpha
    \nonumber
    \\
    \iff &\max_{\alpha\geq0,\beta}\int_{\mathcal{C}}\text{ }p(\bm{c})\min_{q/p\geq 0}\bigg((f(\bm{x},\bm{c})+\beta)\frac{q(\bm{c})}{p(\bm{c})}
    \nonumber
    \\
    &+\alpha \phi\left(\frac{q(\bm{c})}{p(\bm{c})}\right)\bigg)\, d\bm{c}-\beta-\delta\alpha,
    \label{eq-dual-1}
    \end{align}
where (1) is because of the separability of each $q(\bm{c})$. Next, we define conjugate function as $\phi^*(s)=\sup_{x\geq 0}(sx-\phi(x))$ to rewrite Eq.~\eqref{eq-dual-1} as
\begin{align} &\max_{\alpha\geq0,\beta}\int_{\mathcal{C}}\text{ }p(\bm{c})\min_{q/p\geq 0}\bigg((f(\bm{x},\bm{c})+\beta)\frac{q(\bm{c})}{p(\bm{c})}
    \nonumber
    \\
    &+\alpha \phi\left(\frac{q(\bm{c})}{p(\bm{c})}\right)\bigg)\, d\bm{c}-\beta-\delta\alpha,
\nonumber
\\ \iff&\max_{\alpha\geq 0,\beta}-\int_{\mathcal{C}}\text{ } p(\bm{c})\alpha\max_{q/p\geq 0}\bigg(\frac{f(\bm{x},\bm{c})+\beta}{-\alpha}\frac{q(\bm{c})}{p(\bm{c})}
\nonumber
\\
&- \phi\left(\frac{q(\bm{c})}{p(\bm{c})}\right)\bigg)\, d\bm{c}-\beta-\delta\alpha
\nonumber
\\
\iff&\max_{\alpha\geq 0,\beta}-\int_{\mathcal{C}}\text{ } p(\bm{c})\alpha\cdot \phi^*\left(\frac{f(\bm{x},\bm{c})+\beta}{-\alpha}\right)\, d\bm{c}
\nonumber
\\&-\beta-\delta\alpha.
\label{eq-dual-2}
\end{align}
Note here that we denote $0\cdot \phi^*\left(\frac{f(\bm{x},\bm{c})+\beta}{0}\right)$ as 0 because given $\alpha=0$, when $f(\bm{x},\bm{c})+\beta<0$, $\min_{q/p\geq 0}\left((f(\bm{x},\bm{c})+\beta)\frac{q(\bm{c})}{p(\bm{c})}+\alpha \phi\left(\frac{q(\bm{c})}{p(\bm{c})}\right)\right)=-\infty$, and when $f(\bm{x},\bm{c})+\beta\geq0$, $\min_{q/p\geq 0}\left((f(\bm{x},\bm{c})+\beta)\frac{q(\bm{c})}{p(\bm{c})}+\alpha \phi\left(\frac{q(\bm{c})}{p(\bm{c})}\right)\right)=0$. 

Using the fact that when $\phi(x)=|x-1|$, $\phi^*(s)=\max\{s,-1\}$ for $s\leq 1$, $\phi^*(s)=+\infty$ for $s>1$, Eq.~\eqref{eq-dual-2} is equivalent to the following form:
\begin{align}
    &\max_{\alpha\geq 0,\beta}-\int_{\mathcal{C}}\text{ } p(\bm{c})\alpha\cdot \phi^*\left(\frac{f(\bm{x},\bm{c})+\beta}{-\alpha}\right)\, d\bm{c}-\beta-\delta\alpha
    \nonumber
    \\
    \iff&\max_{\alpha\geq 0,\beta}-\beta-\delta\alpha
    \nonumber
    \\
    &-\int_{\mathcal{C}}\text{ } p(\bm{c})\alpha\max\left\{\frac{f(\bm{x},\bm{c})+\beta}{-\alpha},-1\right\}\, d\bm{c}
    \nonumber
    \\
    &\text{s.t. } \frac{f(\bm{x},\bm{c})+\beta}{-\alpha} \leq 1, \forall \bm{c}\in\mathcal{C}
    \nonumber
    \\
    \iff&\max_{\alpha\geq 0,\beta}-\beta-\delta\alpha+\mathbb{E}_{\bm{c}\sim p(\bm{c})}\left[\min\{f(\bm{x},\bm{c})+\beta, \alpha\} \right] 
    \nonumber
    \\
    &\text{s.t. } \alpha+\beta\geq -\inf\limits_{\bm{c}\in\mathcal{C}}f(\bm{x},\bm{c}).
    \label{dual_problem}
\end{align}

The primal problem in Eq.~\eqref{primal_problem} is convex in $q$, and also satisfies Slater's condition because the first constraint can be strictly satisfied with $q=p$. Therefore, the two problems have no dual gap, which leads to the equivalence of the primal problem in Eq.~\eqref{primal_problem} and the dual problem in Eq.~\eqref{dual_problem}.
\end{proof}

\subsubsection{Proof of the Regret Bound of DRBO-KDE}

In this part, we are going to prove Theorem~\ref{main_theorem_2}, which bounds the BCR of DRBO-KDE defined in Eq.~\eqref{BCR}. For the simplicity of formulation, given any function $g$, we define  $q_{\bm{x}}^g:=\arg\min_{q\in\mathcal{B}(\hat{p}_t,\delta_t)}\mathbb{E}_{\bm{c}\sim q}[g(\bm{x},\bm{c})]$. Note that the dependence of $q_{\bm{x}}^g$ on $t$ is implicit here. For example, $q_{\bm{x}_t}^f=\arg\min_{q\in\mathcal{B}(\hat{p}_t,\delta_t)}\mathbb{E}_{\bm{c}\sim q}[f(\bm{x}_t,\bm{c})]$. Then, with a similar idea of the proof for SBO-KDE, we decompose the BCR of DRBO-KDE into the following five terms.
\begin{align}
        &\text{BCR}(T)
    \nonumber
    \\
    =&\mathbb{E}\left[\sum_{t=1}^T\left( \mathbb{E}_{\bm{c}\sim p}[f(\bm{x}^*,\bm{c})]-\mathbb{E}_{\bm{c}\sim p}[f(\bm{x_t},\bm{c})]  \right) \right]
\nonumber
\\
=&\mathbb{E}\left[ \sum_{t=1}^T\left( \mathbb{E}_{\bm{c}\sim p}[f(\bm{x}^*,\bm{c})]  - \mathbb{E}_{\bm{c}\sim \hat{p}_t}[f(\bm{x}^*,\bm{c})]    \right)    \right]
\nonumber
\\
&+\mathbb{E}\left[ \sum_{t=1}^T\left( \mathbb{E}_{\bm{c}\sim \hat{p}_t}[f(\bm{x}^*,\bm{c})]  - \mathbb{E}_{\bm{c}\sim q_{\bm{x}^*}^f}[f(\bm{x}^*,\bm{c})]    \right)    \right] 
\nonumber
\\
&+\mathbb{E}\left[ \sum_{t=1}^T\left( \mathbb{E}_{\bm{c}\sim q_{\bm{x}^*}^f}[f(\bm{x}^*,\bm{c})]  - \mathbb{E}_{\bm{c}\sim q_{\bm{x}_t}^f}[f(\bm{x}_t,\bm{c})]    \right)    \right] 
\nonumber
\\
&+\mathbb{E}\left[ \sum_{t=1}^T\left(  \mathbb{E}_{\bm{c}\sim q_{\bm{x}_t}^f}[f(\bm{x}_t,\bm{c})] -\mathbb{E}_{\bm{c}\sim \hat{p}_t}[f(\bm{x}_t,\bm{c})]\right)    \right] 
\nonumber
\\
&+\mathbb{E}\left[ \sum_{t=1}^T\left( \mathbb{E}_{\bm{c}\sim \hat{p}_t}[f(\bm{x}_t,\bm{c})]  - \mathbb{E}_{\bm{c}\sim p}[f(\bm{x}_t,\bm{c})]    \right)    \right] 
\label{KDEDRBO_BCR_decomposition}
\end{align}
Because $\hat{p}_t\in\mathcal{B}(\hat{p}_t,\delta_t)$ and $q_{\bm{x}_t}^f=\arg\min_{q\in\mathcal{B}(\hat{p}_t,\delta_t)}\mathbb{E}_{\bm{c}\sim q}[f(\bm{x}_t,\bm{c})]$, we have $ \mathbb{E}_{\bm{c}\sim q_{\bm{x}_t}^f}[f(\bm{x}_t,\bm{c})] \leq \mathbb{E}_{\bm{c}\sim \hat{p}_t}[f(\bm{x}_t,\bm{c})]$. So the fourth term of Eq.~\eqref{KDEDRBO_BCR_decomposition} can be bounded by 0. For the second term, because $q_{\bm{x}^*}^f\in\mathcal{B}(\hat{p}_t,\delta_t)$, we have $\|\hat{p}_t(\bm c)-q_{\bm{x}^*}^f(\bm c)\|_1\leq \delta_t$. With the same procedure as the proof of Lemma~\ref{bound_for_1_3_parts_kdesbo_lemma}, we have  $\mathbb{E}_{\bm{c}\sim \hat{p}_t}[f(\bm{x}^*,\bm{c})]  - \mathbb{E}_{\bm{c}\sim q_{\bm{x}^*}^f}[f(\bm{x}^*,\bm{c})] \leq C_4\delta_t$, where $C_4=(abDr\sqrt{\pi}/2+\sqrt{2/\pi})$. The first and fifth terms are the same of those in SBO-KDE, and bounded using Lemma~\ref{bound_for_1_3_parts_kdesbo_lemma}. Therefore, Eq.~\eqref{KDEDRBO_BCR_decomposition} is bounded except for the third item:
\begin{align}
            &\text{BCR}(T)
            \nonumber
            \\
=&\mathbb{E}\left[\sum_{t=1}^T\left( \mathbb{E}_{\bm{c}\sim p}[f(\bm{x}^*,\bm{c})]-\mathbb{E}_{\bm{c}\sim p}[f(\bm{x_t},\bm{c})]  \right) \right]
    \nonumber
    \\
    \leq &\mathbb{E}\left[ \sum_{t=1}^T\left( \mathbb{E}_{\bm{c}\sim q_{\bm{x}^*}^f}[f(\bm{x}^*,\bm{c})]  - \mathbb{E}_{\bm{c}\sim q_{\bm{x}_t}^f}[f(\bm{x}_t,\bm{c})]    \right)    \right]
    \nonumber
    \\
    &+ 2C_1T^{(2+D_c)/(4+D_c)} + \sum_{t=1}^T C_4\delta_t.
\label{kdedrbo_bound_for_1245_terms}
\end{align}
The term left to be bounded is $\mathbb{E}\left[ \sum_{t=1}^T\left( \mathbb{E}_{\bm{c}\sim q_{\bm{x}^*}^f}[f(\bm{x}^*,\bm{c})]  - \mathbb{E}_{\bm{c}\sim q_{\bm{x}_t}^f}[f(\bm{x}_t,\bm{c})]    \right)    \right]$. With the same idea in the proof for SBO-KDE and the same discretization procedure, we decompose it into the following three terms:
    \begin{align}    
&\mathbb{E}\left[ \sum_{t=1}^T\left( \mathbb{E}_{\bm{c}\sim q_{\bm{x}^*}^f}[f(\bm{x}^*,\bm{c})]  - \mathbb{E}_{\bm{c}\sim q_{\bm{x}_t}^f}[f(\bm{x}_t,\bm{c})]    \right)    \right]
\nonumber
\\
  \leq &\mathbb{E}\bigg[ \sum_{t=1}^T\Big( \mathbb{E}_{\bm{c}\sim q_{[\bm{x}^*]_t}^f}[f([\bm{x}^*]_t,\bm{c})]
  \nonumber
  \\
  &\quad\quad- \mathbb{E}_{\bm{c}\sim q_{[\bm{x}^*]_t}^{\text{ucb}_t}}[\text{ucb}_t([\bm{x}^*]_t,\bm{c})] \Big)\bigg]
\nonumber
  \\
&+ \mathbb{E}\bigg[ \sum_{t=1}^T\Big( \mathbb{E}_{\bm{c}\sim q_{\bm{x}_t}^{\text{ucb}_t}}[\text{ucb}_t(\bm{x}_t,\bm{c})]
\nonumber
\\&\quad\quad- \mathbb{E}_{\bm{c}\sim q_{\bm{x}_t}^f}[f(\bm{x}_t,\bm{c})] \Big)\bigg]
\nonumber
\\
&+\mathbb{E}\bigg[ \sum_{t=1}^T\Big( \mathbb{E}_{\bm{c}\sim q_{\bm{x}^*}^f}[f(\bm{x}^*,\bm{c})]
\nonumber
\\
&\quad\quad- \mathbb{E}_{\bm{c}\sim q_{[\bm{x}^*]_t}^f}[f([\bm{x}^*]_t, \bm{c})] \Big)\bigg],
  \label{gp-ucb-regret-decomposition-kdedrbo}
\end{align}
where $\text{ucb}_t(\bm{x}, \bm c)= \mu_t(\bm{x}, \bm c) +\sqrt{\beta_t}\sigma_t(\bm{x}, \bm c)$, and the inequality holds by $\mathbb{E}_{{\bm{c}\sim q_{\bm{x}_t}^{\text{ucb}_t}}}[\text{ucb}_t(\bm{x}_t,\bm{c})] \geq \mathbb{E}_{\bm{c}\sim q^{\text{ucb}_t}_{[\bm{x}^*]_t}}[\text{ucb}_t([\bm{x}^*]_t,\bm{c})]$  as $\bm{x}_t\in\arg\max_{\bm{x}\in\mathcal{X}}\min_{q\in\mathcal{B}(\hat{p}_t,\delta_t)}\mathbb{E}_{\bm{c}\sim q}[\text{ucb}_t(\bm{x},\bm{c})]$.

For the third term of Eq.~\eqref{gp-ucb-regret-decomposition-kdedrbo}, according to the definition of $q_{\bm x^*}^{f}=\arg\min_{q \in \mathcal{B}(\hat{p}_t, \delta_t)}\mathbb{E}_{\bm c \sim q}[f(\bm x^*,\bm c)]$, we have $\mathbb{E}_{\bm{c}\sim q_{\bm{x}^*}^f}[f(\bm{x}^*,\bm{c})]\leq \mathbb{E}_{\bm{c}\sim q_{[\bm{x}]_t^*}^f}[f(\bm{x}^*,\bm{c})]$. Then, the third term of Eq.~\eqref{gp-ucb-regret-decomposition-kdedrbo} can be directly bounded using Lemma~\ref{L_max_lemma}:
\begin{align}
&\mathbb{E}\left[ \sum_{t=1}^T\left( \mathbb{E}_{\bm{c}\sim q_{\bm{x}^*}^f}[f(\bm{x}^*,\bm{c})]- \mathbb{E}_{\bm{c}\sim q_{[\bm{x}^*]_t}^f}[f([\bm{x}^*]_t, \bm{c})] \right)\right]
\nonumber
\\
\leq 
&\mathbb{E}\left[ \sum_{t=1}^T \mathbb{E}_{\bm{c}\sim q_{[\bm{x}^*]_t}^f}\left[f(\bm{x}^*,\bm{c})-f([\bm{x}^*]_t, \bm{c}) \right] \right]
\nonumber
\\
\leq &\mathbb{E}\left[ \sum_{t=1}^T \mathbb{E}_{{\bm{c}\sim q_{[\bm{x}^*]_t}^f}}\left[|f(\bm{x}^*,\bm{c})-f([\bm{x}^*]_t, \bm{c})| \right] \right]
\nonumber
\\
\leq &\mathbb{E}\left[ \sum_{t=1}^T \mathbb{E}_{{\bm{c}\sim q_{[\bm{x}^*]_t}^f}}\left[L_\text{max}D_xr/\tau_t \right] \right]
\nonumber
\\
=&\sum_{t=1}^T\mathbb{E}\left[L_\text{max}\right]D_xr/\tau_t\leq
\sum_{t=1}^T\frac{ab\sqrt{\pi}D_xr}{2}*\frac{2}{t^2D_xabr\sqrt{\pi}}
\nonumber
\\
=&\sum_{t=1}^T\frac{1}{t^2} \leq \frac{\pi^2}{6}.
\label{third_part_gpucb_regret_kdedrbo}
\end{align}

For the first and second terms of Eq.~\eqref{gp-ucb-regret-decomposition-kdedrbo}, according to the definition of $q_{[\bm{x}^*]_t}^f=\arg\min_{q \in \mathcal{B}(\hat{p}_t, \delta_t)}\mathbb{E}_{\bm c \sim q}[f([\bm{x}^*]_t,\bm c)]$ and $q_{\bm x_t}^{\text{ucb}_t}=\arg\min_{q \in \mathcal{B}(\hat{p}_t, \delta_t)}\mathbb{E}_{\bm c \sim q}[\text{ucb}_t(\bm{x}_t,\bm c)]$, we have
$\mathbb{E}_{\bm{c}\sim q_{[\bm{x}^*]_t}^f}[f([\bm{x}^*]_t,\bm{c})] \leq \mathbb{E}_{\bm{c}\sim q_{[\bm{x}^*]_t}^{\text{ucb}_t}}[f([\bm{x}^*]_t,\bm{c})]$ and $\mathbb{E}_{\bm{c}\sim q_{\bm{x}_t}^{\text{ucb}_t}}[\text{ucb}_t(\bm{x}_t,\bm{c})] \leq \mathbb{E}_{\bm{c}\sim q_{\bm{x}_t}^{f}}[\text{ucb}_t(\bm{x}_t,\bm{c})]$. Then, we can give bounds on them with almost the same process in the proof for SBO-KDE. Lemma~\ref{bound_first_part_gpucb_regret_kdedrbo} and Lemma~\ref{kdedrbo_second_parts_gpu_ucb_regret} give the bounds of the first and second terms of Eq.~\eqref{gp-ucb-regret-decomposition-kdedrbo}, respectively.

\begin{lemma}
    Let $\beta_t=2\log{\frac{t^2(\tau_t)^{D_x}}{\sqrt{2\pi}}}$ with $\tau_t={t^2D_xabr\sqrt{\pi}/2}$. At iteration $t$, for $\textnormal{ucb}_t(\bm{x}, \bm c)= \mu_t(\bm{x}, \bm c) +\sqrt{\beta_t}\sigma_t(\bm{x}, \bm c)$, the following holds
    \begin{align}
        \mathbb{E}\bigg[ \sum_{t=1}^T\Big( &\mathbb{E}_{\bm{c}\sim q_{[\bm{x}^*]_t}^f}[f([\bm{x}^*]_t,\bm{c})]
        \nonumber
        \\
        &- \mathbb{E}_{\bm{c}\sim q_{[\bm{x}^*]_t}^{\textnormal{ucb}_t}}[\textnormal{ucb}_t([\bm{x}^*]_t,\bm{c})] \Big)\bigg] \leq \frac{\pi^2}{6}
\label{bound_first_part_gpucb_regret_kdedrbo_ineq}    
\end{align} \label{bound_first_part_gpucb_regret_kdedrbo}    
\end{lemma}
\begin{proof}
    At iteration $t$, we have 
\begin{align}
&\mathbb{E}\left[  \mathbb{E}_{\bm{c}\sim q_{[\bm{x}^*]_t}^f}[f([\bm{x}^*]_t,\bm{c})]- \mathbb{E}_{\bm{c}\sim q_{[\bm{x}^*]_t}^{\text{ucb}_t}}[\text{ucb}_t([\bm{x}^*]_t,\bm{c})] \right]
\nonumber
\\
&\leq 
\mathbb{E}\left[  \mathbb{E}_{\bm{c}\sim q_{[\bm{x}^*]_t}^{\text{ucb}_t}}[f([\bm{x}^*]_t,\bm{c})]- \mathbb{E}_{\bm{c}\sim q_{[\bm{x}^*]_t}^{\text{ucb}_t}}[\text{ucb}_t([\bm{x}^*]_t,\bm{c})] \right]
\nonumber
\\
&\leq \mathbb{E}\left[ \int_{\mathcal{C}}q_{[\bm{x}^*]_t}^{\text{ucb}_t}(\bm{c})\left(f([\bm{x}^*]_t,\bm{c})-\text{ucb}_t([\bm{x}^*]_t,\bm{c})\right)_+\, d\bm{c} \right]
\nonumber
\\
&\leq \sum_{\bm{x}\in \tilde{\mathcal{X}}_t}\mathbb{E}\left[\int_{\mathcal{C}}q_{[\bm{x}^*]_t}^{\text{ucb}_t}(\bm{c})\left(f(\bm{x},\bm{c})- \text{ucb}_t(\bm{x},\bm{c})\right)_+\, d\bm{c} \right],
\nonumber
\end{align}
where the first inequality holds by $q_{[\bm{x}^*]_t}^f=\arg\min_{q \in \mathcal{B}(\hat{p}_t, \delta_t)}\mathbb{E}_{\bm c \sim q}[f([\bm{x}^*]_t,\bm c)]$ and $q_{[\bm{x}^*]_t}^{\text{ucb}_t} \in \mathcal{B}(\hat{p}_t, \delta_t)$. Note that $(X)_+:=\max \{0,X\}$. The remaining proof process is almost identical to that of Lemma~\ref{bound_first_part_gpucb_regret_kdesbo} except that $\hat{p}_t$ is replaced with $q_{[\bm{x}^*]_t}^{\text{ucb}_t}$, so we omit the full proof here.
\end{proof}
\begin{lemma}
        Let $\beta_t=2\log{\frac{t^2(\tau_t)^{D_x}}{\sqrt{2\pi}}}$ with $\tau_t={t^2D_xabr\sqrt{\pi}/2}$, and $h_t^{(i)}=\Theta\left(t^{-1/(4+D_c)}\right)\forall i\in [D_c]$. At iteration $t$, for $\textnormal{ucb}_t(\bm{x},\bm c)= \mu_t(\bm{x}, \bm c) +\sqrt{\beta_t}\sigma_t(\bm{x}, \bm c)$, the following holds
    \begin{align}
       &\mathbb{E}\left[ \sum_{t=1}^T\left( \mathbb{E}_{\bm{c}\sim q_{\bm{x}_t}^{\textnormal{ucb}_t}}[\textnormal{ucb}_t(\bm{x}_t,\bm{c})]- \mathbb{E}_{\bm{c}\sim q_{\bm{x}_t}^f}[f(\bm{x}_t,\bm{c})] \right)\right]
       \nonumber
      \\ &\leq\sqrt{\beta_T\gamma_TC_2}\left(\sqrt{T^{D_c/(4+D_c)}}+\sqrt{T}+\sqrt{C_3\sum_{t=1}^T\delta_t^2}\right),
\label{kdedrbo_second_parts_gpu_ucb_regret_ineq}
    \end{align}
    where $C_2, C_3$ are positive constants, $\gamma_T=\max_{|\mathcal{D}|=T} I(\bm{y}_\mathcal{D},  \bm{f}_\mathcal{D})$, $I(\cdot,\cdot)$ is the information gain, and $\bm{y}_\mathcal{D}, \bm{f}_\mathcal{D}$ are the noisy and true observations of a data set $\mathcal{D}$, respectively.
\label{kdedrbo_second_parts_gpu_ucb_regret}
\end{lemma}
\begin{proof}
Using the tower property of conditional expectation and Fubini's theorem, we have
\begin{align}
&\mathbb{E}_{\bm{c}\sim q_{\bm{x}_t}^{\text{ucb}_t}}[\text{ucb}_t(\bm{x}_t,\bm{c})]- \mathbb{E}_{\bm{c}\sim q_{\bm{x}_t}^f}[f(\bm{x}_t,\bm{c})]
\nonumber
\\
\leq
&\mathbb{E}_{\bm{c}\sim q_{\bm{x}_t}^{f}}[\text{ucb}_t(\bm{x}_t,\bm{c})]- \mathbb{E}_{\bm{c}\sim q_{\bm{x}_t}^f}[f(\bm{x}_t,\bm{c})]
\nonumber
\\
 =&\mathbb{E}\left[  \int_{\mathcal{C}} q_{\bm{x}_t}^{f}(\bm{c})(\text{ucb}_t(\bm{x}_t,\bm{c})-f(\bm{x}_t,\bm{c}) ) \, d\bm{c}  \right]  
 \nonumber
 \\
=&\mathbb{E}\left[  \mathbb{E}\left[\int_{\mathcal{C}} q_{\bm{x}_t}^{f}(\bm{c})(\text{ucb}_t(\bm{x}_t,\bm{c})-f(\bm{x}_t,\bm{c}) )\, d\bm{c} \mid\mathcal{D}_{t-1}\right]  \right]
 \nonumber
 \\
 =&\mathbb{E}\bigg[  \int_{\mathcal{C}} \mathbb{E}\left[q_{\bm{x}_t}^{f}(\bm{c}) \mid \mathcal{D}_{t-1}\right]
\nonumber
\\
&\quad\quad\cdot\mathbb{E}\left[\text{ucb}_t(\bm{x}_t,\bm{c})-f(\bm{x}_t,\bm{c}) \mid\mathcal{D}_{t-1}\right]\, d\bm{c}   \bigg]
 \nonumber
 \\
=&\mathbb{E}\left[ \int_{\mathcal{C}}q_{\bm{x}_t}^{f}(\bm{c})\sqrt{\beta_t}\sigma(\bm{x}_t,\bm{c})\, d\bm{c} 
    \right],
    \nonumber
\end{align}
where the first inequality holds by $q_{\bm x_t}^{\text{ucb}_t}=\arg\min_{q \in \mathcal{B}(\hat{p}_t, \delta_t)}\mathbb{E}_{\bm c \sim q}[\text{ucb}_t(\bm{x}_t,\bm c)]$ and $ q_{\bm{x}_t}^{f} \in \mathcal{B}(\hat{p}_t, \delta_t)$, the third equality follows from the fact that the distribution of context variable and the GP are independent. By $q_{\bm{x}_t}^{f}(\bm{c}) \in \mathcal{B}(\hat{p}_t, \delta_t)$, we have $\mathbb{E}\left[\|q_{\bm{x}_t}^f(\bm c)-\hat{p}_t(\bm c)\|_1\right] \leq \delta_t$. Using the same process with the proof of Lemma~\ref{bound_for_1_3_parts_kdesbo_lemma}, we have $\mathbb{E}\left[\|p(\bm{c})-\hat{p}_t(\bm{c})\|_1\right] \leq C_0r^{D_c/2}t^{-2/(4+D_c)}$. Therefore, we have $\mathbb{E}\left[\|q_{\bm{x}_t}^f (\bm c)-p(\bm c)\|_1\right] \leq \mathbb{E}\left[\|q_{\bm{x}_t}^f(\bm c)-\hat{p}_t(\bm c)\|_1\right] + \mathbb{E}\left[\|p(\bm c)-\hat{p}_t(\bm c)\|_1\right] \leq \delta_t+C_0r^{D_c/2}t^{-2/(4+D_c)}$.
Then,
\begin{align}
&\sum_{t=1}^T\mathbb{E}\left[ \int_{\mathcal{C}}q_{\bm{x}_t}^{f}(\bm{c})\sqrt{\beta_t}\sigma_t(\bm{x}_t,\bm{c})\, d\bm{c} \right]
 \nonumber
 \\
 = &\sum_{t=1}^{T}\mathbb{E}\left[ \int_{\mathcal{C}}p(\bm{c})\sqrt{\beta_t}\sigma_t(\bm{x}_t,\bm{c})\, d\bm{c} 
    \right] 
    \nonumber
    \\
&+ \sum_{t=1}^{T}\mathbb{E}\left[ \int_{\mathcal{C}}(q_{\bm{x}_t}^{f}(\bm{c})-p(\bm{c}))\sqrt{\beta_t}\sigma_t(\bm{x}_t,\bm{c})\, d\bm{c} 
    \right]
    \nonumber
    \\
    \leq &\mathbb{E}\left[\int_{\mathcal{C}}p(\bm{c}) \sqrt{\sum_{t=1}^T\beta_t\sigma_t^2(\bm{x}_t,\bm{c})}\sqrt{T} \, d\bm{c} \right] 
    \nonumber
    \\
    &+\mathbb{E}\left[\sum_{t=1}^T \|q_{\bm{x}_t}^{f}(\bm{c})-p(\bm{c}) \|_1 \sqrt{\beta}_t \sigma_t(\bm{x}_t,\bm{c}^*_t)\right]
    \nonumber
    \\    \leq&\sqrt{\beta_T\gamma_T T C} 
    \nonumber
    \\
    &+ \mathbb{E}\left[ \sum_{t=1}^T(\delta_t+C_0r^{D_c/2}t^{-2/(4+D_c)})\sqrt{\beta_t}\sigma_t(\bm{x}_t, \bm{c}_t^*) \right]
    \nonumber
    \\
\leq&\sqrt{\beta_T\gamma_T T C} +\mathbb{E}\Bigg[ \Bigg(C_0r^{D_c/2}\sqrt{\sum_{t=1}^Tt^{-4/(4+D_c)}} \nonumber
\\
&+\sqrt{\sum_{t=1}^T\delta_t^2}\Bigg)\sqrt{\sum_{t=1}^T \beta_t\sigma_t^2(\bm{x}_t, \bm{c}_t^*)} \Bigg]
\nonumber
\\
\leq&\sqrt{\beta_T\gamma_T T C} + \mathbb{E}\Bigg[ \bigg(C_0r^{D_c/2}
\nonumber
\\
&\cdot\sqrt{((4+D_c)/D_c)T^{D_c/(4+D_c)}}+\sqrt{\sum_{t=1}^T\delta_t^2}\bigg)\sqrt{\beta_T\gamma_TC} \Bigg]
\nonumber
\\
\leq&\sqrt{\beta_T\gamma_TC_2}\bigg(\sqrt{T^{D_c/(4+D_c)}}+\sqrt{T}
\nonumber
\\
&+\sqrt{C/C_2\sum_{t=1}^T\delta_t^2}\bigg),
\nonumber
\end{align}
where $c_t^*=\arg\max_{\bm{c}\in\mathcal{C}}\sigma_t(\bm{x}_t,\bm{c})$, $C=8/\log{(1+\sigma^2)}$,  $\sigma$ is the standard deviation of observation noise, the first inequality follows from Cauchy inequality $\sum_{t=1}^T \sqrt{\beta_t}\sigma_t(\bm x _t,\bm c)\leq \sqrt{\sum_{t=1}^T\beta_t\sigma_t^2(\bm x_t,\bm c)}\sqrt{T}$, and Hölder inequality $\int_{\mathcal{C}} (q_{\bm x _t}^f(\bm c)-p(\bm c))\sigma_t(\bm x_t,\bm c)\, d\bm c \leq  \|q_{\bm x _t}^f(\bm c)-p(\bm c)\|_1 \sigma_t(\bm x_t,\bm c_t^*)$, the second inequality follows from Lemma 5.4 of~\cite{gpucb}, i.e., $\sum_{t=1}^T \beta_t\sigma_t^2(\bm x_t, \bm c)\leq \beta_T\gamma_TC$, and the $\ell_1$ bound $\mathbb{E}\left[\|q_{\bm{x}_t}^f (\bm c)-p(\bm c)\|_1\right] \leq \delta_t+C_0r^{D_c/2}t^{-2/(4+D_c)}$, the third inequality also follows from Cauchy inequality $\sum_{t=1}^T t^{-2/(4+D_c)}\sqrt{\beta_t}\sigma_t(\bm x_t,\bm c_t^*) \leq \sqrt{\sum_{t=1}^Tt^{-4/(4+D_c)}}$ $\sqrt{\sum_{t=1}^T \beta_t\sigma_t^2(\bm{x}_t, \bm{c}_t^*)}$ and $\sum_{t=1}^T \delta_t\sqrt{\beta_t}\sigma_t(\bm x_t,\bm c_t^*) \leq \sqrt{\sum_{t=1}^T \delta_t^2}\sqrt{\sum_{t=1}^T \beta_t\sigma_t^2(\bm{x}_t, \bm{c}_t^*)}$, the fourth inequality also follows from Lemma 5.4 of~\cite{gpucb}, i.e., $\sum_{t=1}^T\beta_t\sigma_t^2(\bm x_t, \bm c_t^*)\leq \beta_T\gamma_TC$, and the last inequality holds by letting $C_2=C\max\{1,C_0^2r^{D_c}(4+D_c)/D_c\}$. Let $C_3 = C/C_2$, implying that the lemma holds.
\end{proof}

By now, we can directly get the upper bound on the BCR of DRBO-KDE, which is re-stated in Theorem~\ref{main_theorem2_appendix} for clearness, by directly summing over all components of Eq.~\eqref{KDEDRBO_BCR_decomposition} that is bounded.

\begin{theorem}[Theorem~\ref{main_theorem_2} in the main paper]
Let $\beta_t = 2\log(t^2/\sqrt{2\pi}) +2D_x\log(t^2D_xabr\sqrt{\pi}/2)$. With the underlying PDF $p(\bm{c})$ satisfying the condition in Lemma~\ref{l2_bound_lemma}, $\hat{p}_t(\bm{c})$ defined as Eq.~\eqref{KDE} and $h_t^{(i)}=\Theta\left(t^{-1/(4+D_c)}\right)\forall i\in [D_c]$, the BCR of DRBO-KDE satisfies
\begin{align}
\textnormal{BCR}(T)\leq&\frac{\pi^2}{3} + \sqrt{\beta_T\gamma_TC_2}\bigg(\sqrt{T^{D_c/(4+D_c)}}+\sqrt{T}
\nonumber
\\
&+\sqrt{C_3\sum_{t=1}^T\delta_t^2}\bigg) + 2C_1T^{\frac{2+D_c}{4+D_c}} + \sum_{t=1}^T C_4\delta_t,
\nonumber
\end{align}
where $C_1,C_2,C_3,C_4>0$ are constants, $\gamma_T=\max_{|\mathcal{D}|=T} I(\bm{y}_\mathcal{D},  \bm{f}_\mathcal{D})$, $I(\cdot,\cdot)$ is the information gain, and $\bm{y}_\mathcal{D}, \bm{f}_\mathcal{D}$ are the noisy and true observations of a data set $\mathcal{D}$, respectively.
\label{main_theorem2_appendix}
\end{theorem}
\begin{proof}
    With Eq.~\eqref{KDEDRBO_BCR_decomposition} and Eq.~\eqref{gp-ucb-regret-decomposition-kdedrbo}, BCR$(T)$ has the following upper bound:
    \begin{align}
     &\textnormal{BCR}(T)
     \nonumber
     \\
     \leq&\mathbb{E}\left[ \sum_{t=1}^T\left( \mathbb{E}_{\bm{c}\sim p}[f(\bm{x}^*,\bm{c})]  - \mathbb{E}_{\bm{c}\sim \hat{p}_t}[f(\bm{x}^*,\bm{c})]    \right)    \right]   
     \nonumber
     \\
     & +\mathbb{E}\left[ \sum_{t=1}^T\left( \mathbb{E}_{\bm{c}\sim \hat{p}_t}[f(\bm{x}^*,\bm{c})]  - \mathbb{E}_{\bm{c}\sim q_{\bm{x}^*}^f}[f(\bm{x}^*,\bm{c})]    \right)    \right]
     \nonumber
     \\
     & +\mathbb{E}\left[ \sum_{t=1}^T\left(  \mathbb{E}_{\bm{c}\sim q_{\bm{x}_t}^f}[f(\bm{x}_t,\bm{c})] -\mathbb{E}_{\bm{c}\sim \hat{p}_t}[f(\bm{x}_t,\bm{c})]\right)    \right] 
     \nonumber
     \\
     & +\mathbb{E}\left[ \sum_{t=1}^T\left( \mathbb{E}_{\bm{c}\sim \hat{p}_t}[f(\bm{x}_t,\bm{c})]  - \mathbb{E}_{\bm{c}\sim p}[f(\bm{x}_t,\bm{c})]    \right)    \right]
     \nonumber
     \\
     & +\mathbb{E}\bigg[ \sum_{t=1}^T\Big( \mathbb{E}_{\bm{c}\sim q_{[\bm{x}^*]_t}^f}[f([\bm{x}^*]_t,\bm{c})]
     \nonumber
     \\
     &\quad\quad- \mathbb{E}_{\bm{c}\sim q_{[\bm{x}^*]_t}^{\text{ucb}_t}}[\text{ucb}_t([\bm{x}^*]_t,\bm{c})] \Big)\bigg]
     \nonumber
     \\
     & +\mathbb{E}\left[ \sum_{t=1}^T\left( \mathbb{E}_{\bm{c}\sim q_{\bm{x}_t}^{\text{ucb}_t}}[\text{ucb}_t(\bm{x}_t,\bm{c})]- \mathbb{E}_{\bm{c}\sim q_{\bm{x}_t}^f}[f(\bm{x}_t,\bm{c})] \right)\right]
     \nonumber
     \\
     & +\mathbb{E}\left[ \sum_{t=1}^T\left( \mathbb{E}_{\bm{c}\sim q_{\bm{x}^*}^f}[f(\bm{x}^*,\bm{c})]- \mathbb{E}_{\bm{c}\sim q_{[\bm{x}^*]_t}^f}[f([\bm{x}^*]_t, \bm{c})] \right)\right].
     \nonumber
    \end{align}
The sum of the first four terms has an upper bound of $2C_1T^{(2+D_c)/(4+D_c)}+\sum_{t=1}^TC_4\delta_t$ by Eq.~\eqref{kdedrbo_bound_for_1245_terms}. The fifth term has a bound of $\pi^2/6$ by Eq.~\eqref{bound_first_part_gpucb_regret_kdedrbo_ineq}. The sixth term has a bound of $\sqrt{\beta_T\gamma_TC_2}\left(\sqrt{T^{D_c/(4+D_c)}}+\sqrt{T}+\sqrt{C_3\sum_{t=1}^T\delta_t^2}\right)$ by Eq.~\eqref{kdedrbo_second_parts_gpu_ucb_regret_ineq}. The last term has a bound of $\pi^2/6$ by Eq.~\eqref{third_part_gpucb_regret_kdedrbo}. By summing all these up, the theorem holds.
\end{proof}

\section{Details of Experimental Setting}
\label{experimental_setting_appendix}

\subsection{Method Implementation}
\label{method_appendix}
The acquisition functions for all algorithms are implemented and optimized using BoTorch~\cite{botorch}. The inner convex optimization problems of DRBO-KDE and DRBO-MMD are solved by CVXPY~\cite{cvxpy}. For all algorithms, we set the exploration-exploitation trade-off parameter of UCB as $\sqrt{\beta_t}=1.5$. When optimizing acquisition functions in BoTorch, the number of samples for initialization "raw\_samples" is set to 1,024 except for DRBO-KDE and DRBO-MMD, which are set to 32 due to their slow inner optimization speed, and the number of starting points for multi-start acquisition function optimization "num\_restarts" is set to 50 except for DRBO-KDE and DRBO-MMD, which are set to 5 for the same reason. The other hyper-parameters and details for each algorithm are summarized as follows.
\begin{itemize}
\item \textbf{SBO-KDE.} For the number of Monte Carlo samples for the SAA method, we set $M=1,024$. We choose the Gaussian kernel $K(\bm x)=\frac{1}{(2\pi)^{D_c/2}}e^{-\|\bm x\|_2^2}$ for KDE. The bandwidth $h_t^{(i)}=(\frac{4}{D_c+2})^{\frac{1}{D_c+4}} \hat{\sigma}_t^{(i)}t^{-1/(4+D_c)}$ based on the rule of thumb~\cite{kde_1}, where $\hat{\sigma}_t^{(i)}$ is the standard deviation of the $i$th dimension of observed context.
\item \textbf{DRBO-KDE.} The radius of the distribution set is set as $\delta_t=t^{-2/(4+D_c)}$, which can guarantee a sub-linear regret as we introduced in Section~\ref{sec-DRBO-KDE}. The kernel and bandwidth for KDE are the same as those used in SBO-KDE. For the optimization of $\inf_{\bm{c}\in\mathcal{C}}\text{ucb}_t(\bm{x},\bm{c})$ for the calculation of $S(\text{ucb}_t)$ in Eq.~\eqref{eq-dro-total-saa}, we use a Sobol sequence of size 1,024 and select the minimum value over them. The other hyper-parameters are the same as SBO-KDE.
\item 
\textbf{DRBO-MMD.} The continuous context space is discretized into a space 
 of size $|\tilde{C}|=\lceil 100^{1/D_c} \rceil^{D_c}$. We don't set $|\tilde{C}|$ too big because a larger size would significantly enlarge the running time. The reference probability distribution $\hat{P}_t$ is defined as the empirical distribution in the discretized context space $\tilde{C}$. The radius of the distribution set is set as $\delta_t=\frac{1}{\sqrt{t}}\left(2+\sqrt{2\log{1/\gamma}}\right)$ with $\gamma=0.1$ as stated in Lemma 3 in~\cite{drbo}.
\item \textbf{DRBO-MMD-MinimaxApprox.} The continuous context space is discretized into a larger space 
 of size $|\tilde{C}|=\lceil 1024^{1/D_c} \rceil^{D_c}$, because it accelerates DRBO-MMD with minimax approximation. The radius $\delta_t$ of the distribution set and the reference probability distribution $\hat{P}_t$ are the same as MMD.
\item \textbf{StableOpt.} There is no standard way of choosing $\mathcal{C}_t$. Instead of setting $\mathcal{C}_t=\mathcal{C}$, which is the worst case under the full context space and too conservative, we let each dimension of $\mathcal{C}_t$ be $[\mu_{\bm{c}}^{(i)}-\sigma_{\bm{c}}^{(i)}, \mu_{\bm{c}}^{(i)}+\sigma_{\bm{c}}^{(i)}]$, where $\mu_{\bm{c}}^{(i)}$ and $\sigma_{\bm{c}}^{(i)}$ are the mean and standard deviation of the $i$th dimension of observed contexts, respectively. For the inner optimization loop, we use a Sobol sequence of size 1,024 instead of using other optimizing strategys like L-BFGS. That's because other methods are too slow for this two-layers optimization.
\item \textbf{GP-UCB.} No specific hyper-parameters for GP-UCB except $\sqrt{\beta_t}=1.5$.
\end{itemize}

\subsection{Problem Definition}
\label{problem_definition_appendix}
In this part, we provide the detailed definition of the problems employed in the experiments.
\begin{itemize}
\item \textbf{Ackley.} The Ackley function in the experiment is a three dimensional function with two dimensional decision variable $\bm{x}\in[0,1]^2$ and one dimensional context variable $c\in[0,1]$:
\begin{align}
f(\bm{x},c)=&ae^{-b\sqrt{\frac{1}{3}(\sum_{i=1}^2\tilde{x}_i^2+\tilde{c}^2})}
\nonumber
\\
&-e^{\frac{1}{3}(\sum_{i=1}^2\cos{h\tilde{x}_i}+\cos{h\tilde{c}})}+a+e,
    \nonumber
\end{align}
where $a=20,b=0.2,h=2\pi,\tilde{x}_i = 65.536x_i-32.768, \tilde{c} = 65.536c-32.768$.

The distribution of $c$ is defined as $\mathcal{N}(0.5, 0.15^2)$, and $c$ is clipped to $[0,1]$ when it is out of the bound.
\item \textbf{Modified Branin.} The modified Branin function is modified from~\cite{modified_branin}. The original Branin function is defined as
\begin{equation}
    h(u,v)=a(v-bu^2+cu-r)^2+s(1-t)\cos{(u)}+s,
    \nonumber
\end{equation}
where $a=1, b=5.1/(4\pi^2),c=5/\pi,r=6,s=10,t=1/(8\pi).$

The Modified Branin function in the experiment is a four dimensional function with two dimensional decision variable $\bm{x}\in[0,1]^2$ and two dimensional context variable $\bm{c}\in[0,1]^2$:
\begin{equation}
f(\bm{x},\bm{c})=-\sqrt{h(15x_1-5,15c_1)h(15c_2-5,15x_2)}.
    \nonumber
\end{equation}
The distributions of both dimensions of $\bm{c}$ are defined as $\mathcal{N}(0.5, 0.1^2)$, and $\bm{c}$ is clipped to $[0,1]$ when they are out of the bound. 

\item \textbf{Hartmann.} The Hartmann function in the experiment is a six dimensional function with five dimensional decision variable $\bm{x}\in[0,1]^5$ and one dimensional context variable $c\in[0,1]$:
\begin{equation}
    f(\bm{x},c)=\sum_{i=1}^4\alpha_ie^{-\sum_{j=1}^6A_{ij}(y_j-P_{ij})^2},
    \nonumber
\end{equation}
where $c=y_6,$ $\bm{x}=(y_1,\dots,y_5)^{\mathrm{T}},$ $\bm{\alpha}=(1.0,2.0,3.0,3.2)^{\mathrm{T}}$,
\begin{equation}
\mathbf{A}=\left(
\begin{array}{cccccc}
    10 & 3 & 17 & 3.50 &1.7 & 8
     \\
     0.05 & 10 & 17 & 0.1 & 8 & 14
     \\
     3 & 3.5 & 1.7 & 10 & 17 & 8
     \\
     17 & 8 & 0.05 & 10 & 0.1 & 14
\end{array}
\right),
    \nonumber
\end{equation}
\begin{align}
&\mathbf{P}=
\nonumber
\\
&10^{-4}\left(
\begin{array}{cccccc}
    1312 & 1696 & 5569 & 124 & 8283 & 5886
     \\
     2329 & 4135 & 8307 & 3736 & 1004 & 9991
     \\
     2348 & 1451 & 3522 & 2883 & 3047 & 6650
     \\
     4047 & 8828 & 8732 & 5743 & 1091 & 381
\end{array}
\right).
    \nonumber
\end{align}

The distribution of $c$ is defined as $\mathcal{N}(0.5, 0.1^2)$, and $c$ is clipped to $[0,1]$ when it is out of the bound. 

\item \textbf{Complicated Hartmann.} The function definition is the same as the Hatrmann function with five dimensin decision variable $\bm x \in [0,1]^5$ and one dimensional context variable $c\in[0,1]$.

The distribution of $c$ now is a more complicated distribution than that of Hartmann, which is defined as a mixture of six normal distributions and two Cauchy distributions: $\mathcal{N}(0.1,0.02^2),\mathcal{N}(0.3,0.075^2),\mathcal{N}(0.4,0.1^2),$ $\mathcal{N}(0.5, 0.1^2),\mathcal{N}(0.7, 0.075^2),\mathcal{N}(0.8, 0.03^2)$ and $C(0.2, 0.02),C(0.8, 0.02)$, where $C(x_0,\gamma)$ denotes the Cauchy distribution, and $c$ is clipped to $[0,1]$ when it is out of the bound.

\item \textbf{Newsvendor Problem.} The newsvendor problem from~\cite{simoptgithub} in the experiment is a two dimensional function with one dimensional decision variable $\bm x\in[0,1]$ (initial inventory) and one dimensional context variable $\bm c \in [0,1]$ (customer demand). The function value is defined as $f(x,c) = 9\min\{x,c\}+\max\{0,x-c\}-5x$. The customer demand $c$ follows a Burr Type XII distribution with PDF $p(c;\alpha,\beta)=\alpha\beta\frac{c^{\alpha-1}}{(1+c^{\alpha})^{\beta+1}}$, where $\alpha=2$ and $\beta=20$, and $c$ is clipped to $[0,1]$ when it is out of the bound. 

\item \textbf{Portfolio Optimization.} The portfolio optimization problem in the experiment is a five dimensional function with three  dimensional decision variable $\bm{x}\in[0,1]^3$ (risk and trade aversion parameters, and holding cost multiplier), and two dimensional context variable $\bm{c}\in[0,1]^2$ (bid-ask spread and borrowing cost). The objective function is computed as the posterior mean of a GP trained on 3,000 samples generated from the CVXPortfolio~\cite{portfolio} by~\cite{risk1}. Here, the scale of the variables are all re-scaled to $[0,1]$ by~\cite{risk1} in the data they provided.

The distribution of context variable is defined in two ways, corresponding to two different problems. The first is that both dimensions of $\bm{c}$ are subject to uniform distribution $\mathrm{U}(0,1)$. The second is that the distributions of both dimensions of $\bm{c}$ are defined as $\mathcal{N}(0.5, 0.1^2)$ and $\bm{c}$ is clipped to $[0,1]$ when they are out of the bound.
\end{itemize}

\section{Computational Complexity Comparison}
\label{ComplexityAppendix}
In this section, we provide the computational complexity comparison of the algorithms in the experiments.

The computational complexity of fitting a GP model with $t$
 evaluated points is $\mathcal{O}(t^3+t^2d)$
 for all BO algorithms, where $d$
 is the dimension of the function. The major difference among their computational complexity lies in the evaluation time of the acquisition functions. For GP-UCB, it takes $\mathcal{O}(t^2+td)$
 time for each evaluation. Each evaluation of DRBO-MMD and DRBO-MMD-MinimaxApprox incurs $\mathcal{O}((t^2+td)|\mathcal{C}|+|\mathcal{C}|^3)$
 and $\mathcal{O}((t^2+td)|\mathcal{C}|+|\mathcal{C}|^2)$
 time, respectively, as shown in~\cite{drbo_wcs}. Note that $|\mathcal{C}|$
 denotes the size of the discretized space $\mathcal{C}$. For StableOpt, it takes $\mathcal{O}(n_{\text{sobol}}(t^2+td))$ time for each evaluation, where $n_{\text{sobol}}$
 denotes the size of Sobol sequence used to calculate the minimal value of UCB and is set to $1024$
 in our experiments. For SBO-KDE, the time of each evaluation is $\mathcal{O}((t^2+td)M)$, because $M$ 
 Monte Carlo samples are used for estimating the acquisition function in SAA. For DRBO-KDE, each evaluation requires solving an inner optimization problem, which is transformed into a two dimensional convex optimization problem as shown in Proposition~\ref{dro_transformation_prop}. To calculate the constraint $S(\text{ucb}_t)$
 of the transformed problem, it takes $\mathcal{O}(n_{\text{sobol}}(t^2+td))$
 time, where $n_{\text{sobol}}$
 is also set to $1024$ in our experiments. For the transformed two dimensional convex optimization, it incurs up to $\mathcal{O}((t^2+td)M)$
 time using the interior point method, where $M$
 is the number of Monte Carlo samples. Hence, the complexity of each evaluation of DRBO-KDE is $\mathcal{O}((t^2+td)(M+n_{\text{sobol}}))$. For SBO-KDE and DRBO-KDE, $M$ is set to 
$1024$.

From the above analysis, we can find that GP-UCB has the lowest complexity. DRBO-MMD is very expensive in terms of $|\mathcal{C}|$
 while DRBO-MMD-MinimaxApprox is much cheaper. StableOpt and SBO-KDE have a similar time complexity, while DRBO-KDE is a bit more expensive than SBO-KDE due to the inner two dimensional convex optimization. In Table~\ref{table1}, we show the average runtime (in seconds) for $100$ evaluations of 10 independent runs of each algorithm on the Ackley function, using the same setting of the optimizer in BoTorch. We can observe that the runtime order in the table is consistent with our computational complexity analysis.

\begin{table}[h]
\begin{center}
\begin{tabular}{|c | c|} 
 \hline
 \textbf{Algorithm} & \textbf{Time}  \\  
 \hline
 GP-UCB	 & $11.35\pm 0.40$  \\ 
 \hline
 DRBO-MMD	 & $5254.65\pm 356.32 $ \\
 \hline
 DRBO-MMD-MinimaxApprox	 & $119.04\pm 2.61$\\
 \hline
 StableOpt	 & $32.75\pm 1.35$ \\
 \hline
 SBO-KDE	 &$ 30.57\pm 1.06$\\ 
 \hline
 DRBO-KDE	 &$ 180.61\pm 6.90$\\ 
 \hline
\end{tabular}
\caption{Comparison on the runtime (in seconds) for 100 evaluations on the Ackley function, which is averaged over 10 independent runs.}
\label{table1}
\end{center}
\end{table}

\section{Experiments with Higher Context Dimension}
\label{AppendixD}
We use a small dimension of context variable in the main paper by following the experiments in DRBO literature, where the dimension of context variable tends to be relatively low (at most three)~\cite{drbo,drbo_wcs}. In fact, the context variable has a relatively low dimension in many real-world problems, such as customer demand in inventory management. Here, we conduct one more experiment on a problem with higher dimensional context variable, which is the Ackley function with 2 dimensional decision variable and 4 dimensional context variable. The distribution of all dimensions of context variable $\boldsymbol{c}$ is defined as $\mathcal{N}(0.5, 0.15^2)$, and $\boldsymbol{c}$ is clipped to $[0,1]^4$ when it is out of the bound. The result is presented in Figure~\ref{figackley}, showing that SBO-KDE and DRBO-KDE still outperform other methods.

\begin{figure}[h]
    \centering \includegraphics[width=0.7\linewidth]{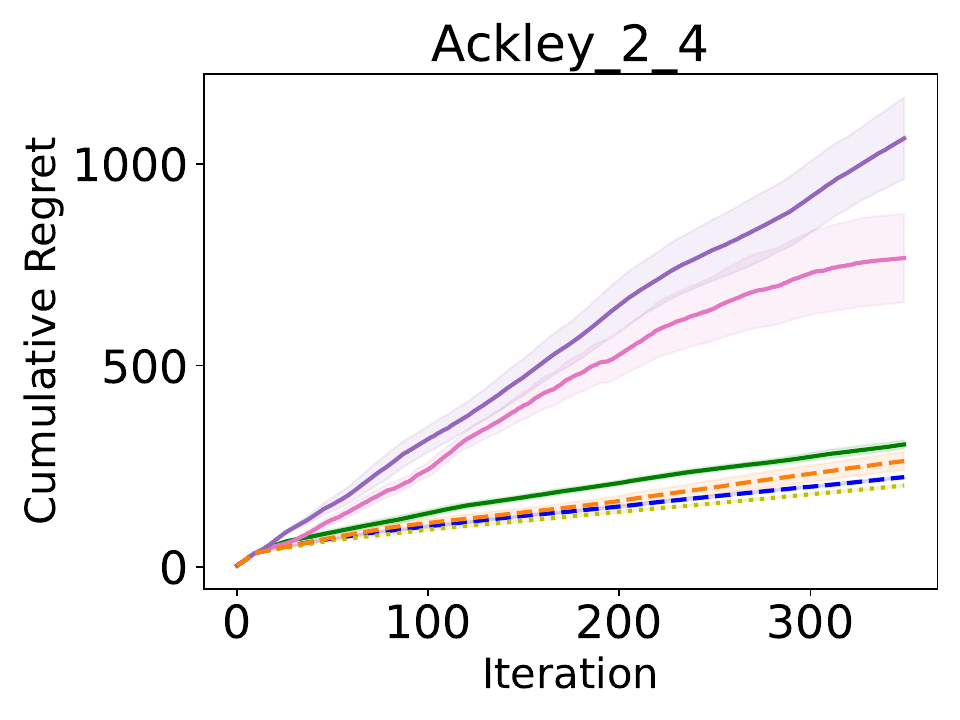}
\includegraphics[width=1.0\linewidth]{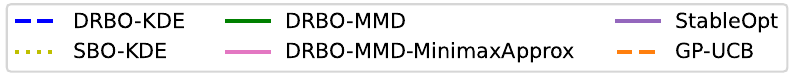}
    \caption{Mean and standard error of cumulative regret on Ackley function with 2 dimensional decision variable and 4 dimensional context variable.}
    \label{figackley}
\end{figure}

\section{Analysis of Complexity of Context Variable}
\label{AppendixE}
\begin{table}[h]
\begin{center}
\begin{tabular}{|c | c| c|} 
 \hline
 \textbf{\#Samples} & \textbf{Normal} & \textbf{Complicated} \\  
 \hline
 10		 & $0.3222 \pm 0.0327$  & $0.6216 \pm 0.0190$\\ 
 \hline
 100	 & $0.1315\pm 0.0088 $ &$0.5380 \pm 0.0333$ \\
 \hline
 200	 & $0.1156\pm 0.0079$& $0.4853 \pm 0.0233$\\
 \hline
 300	 & $0.0962 \pm 0.0083$ & $0.4724 \pm 0.0258$\\
 \hline
\end{tabular}
\caption{Average total variation of 20 independent runs under different number of samples between PDF estimated by KDE and the true context distribution in the experiments on Hartmann function.}
\label{table2}
\end{center}
\end{table}

We introduce DRBO-KDE considering that the estimated PDF may have high estimation error when the true distribution is complicated. To strengthen the motivation, in this section, we report the discrepancy (measured by total variation) between PDF estimated by KDE and the true context distribution in the experiments of Hartmann function with two distributions, i.e., the normal distribution and the complicated distribution. Table~\ref{table2} gives the average results of 20 independent runs under different number of samples for KDE. We can observe a greater discrepancy under the complicated distribution. The last two sub-figures of Figure1(a) in the main paper have shown that SBO-KDE outperforms DRBO-KDE for the normal distribution, while vice versa for the complicated distribution. Thus, these results validate our motivation of introducing DRBO-KDE.

\end{document}